\documentclass[10pt,twocolumn,letterpaper]{article}
\usepackage[textwidth=6.75in,textheight=9in,columnsep=0.25in,centering]{geometry}
\usepackage[T1]{fontenc}
\usepackage{times}
\usepackage{microtype}
\usepackage{amsmath,amssymb,amsthm}
\usepackage{graphicx,booktabs,xcolor}
\usepackage{balance}
\usepackage{algorithm,algorithmic}
\usepackage{tikz,pgfplots}
\usetikzlibrary{positioning,fit,backgrounds,arrows.meta,calc,decorations.pathreplacing}
\pgfplotsset{compat=1.17}
\usepackage[round,authoryear]{natbib}
\usepackage[font=small,labelfont=bf]{caption}
\usepackage{xurl}
\usepackage[hidelinks]{hyperref}
\newtheorem{proposition}{Proposition}
\definecolor{mlipsblue}{RGB}{37,103,153}
\definecolor{mlipsorange}{RGB}{202,112,42}
\definecolor{mlipsgreen}{RGB}{55,130,98}
\definecolor{mlipspurple}{RGB}{124,85,154}
\pgfplotsset{mlipsaxis/.style={
  tick label style={font=\scriptsize},
  label style={font=\small},
  axis line style={black!50},
  grid=major,grid style={black!10},
  axis background/.style={fill=white},
  scaled ticks=false,
  /pgf/number format/1000 sep={}}}

\title{MLIPS: Multi-Layer Indexer Pool Sharing}
\author{Samkit Jain \and Sai Kishan Pampana \and Shlok Agarwal \and Yashwant Marathe \and Bhargava Gopireddy \and Aradhana Mohan Parvathy \and Sachin Idgunji}
\date{NVIDIA Corporation\\
\small\texttt{\{samkitj,spampana,shloka,ymarathe,bgopireddy,}\\
\small\texttt{amohanparvat,sidgunji\}@nvidia.com}}

\begin{document}
\maketitle
\begin{abstract}
% Source fragment: sections/abstract.tex
Sparse attention limits the keys used by each query, but finding those keys can still require
a full-prefix scan at every layer. We introduce \textbf{MLIPS (Multi-Layer Indexer Pool
Sharing)}, a training-free method that shares candidate search while retaining each
consumer's own scoring and selection. A producer publishes a top-$M$ pool larger than
the attention budget $K$; subsequent layers select their own top-$K$ from that pool.
Measurements across three sparse-attention model families show substantially higher
coverage from broad pools than overlap between final selections. We implement layerwise
sharing, persistent decode reuse, and a coarse producer in vLLM for DeepSeek-V4-Flash.
Across four long-context suites, three consumer layers per producer give
$1.14$--$1.40\times$ speedups and $14.8$--$27.4\%$ lower GPU energy, with observed accuracy
changes of $-1.49$ to $+0.29$ percentage points in single runs. More extensive sharing
exposes task-dependent quality losses. Sampled prefill traces also reveal recurring selected
KV positions across query groups and layers. We characterize this shared working-set
structure and use it to motivate Top-$G$, an exploratory bounded pool construction for
prefill. Its deployment, task accuracy, and serving benefits remain open.
\end{abstract}

% Source fragment: sections/introduction.tex
\section{Introduction}
\label{sec:intro}

Sparse attention makes long-context language models more efficient by restricting each
query to a small set of keys. Yet selecting those keys can remain expensive. In DeepSeek
Sparse Attention, a Lightning Indexer scores the visible prefix before choosing the top-$K$
positions used by attention \citep{deepseekv32,deepseekv4}. Attention then operates on a
bounded selection, while the indexer's search grows with context length. Repeating that
search across sparse layers and queries leaves substantial work to optimize.

Sharing selections is one established response. IndexCache lets intervening layers inherit
a retained indexer's top-$K$, with training-free calibration or training-aware distillation
\citep{indexcache}. DeepSeek-V4.1-Flash trains deeper indexers to re-select from a larger
pool produced by an earlier layer \citep{deepseekv41}. These approaches motivate a
deployment question: how much candidate search can an existing checkpoint share without
retraining its indexers or changing its attention operator?

Our starting observation is that \emph{agreement on broad candidate sets is much stronger
than agreement on final selections}. On DeepSeek-V4-Flash traces spanning contexts up to
one million tokens, an adjacent producer's top-$512$ contains $23$--$34\%$ of a consumer's
independent top-$512$, depending on the suite. A $16{,}384$-entry producer pool raises that
coverage to $72$--$85\%$. The same distinction appears in MiniMax-M3 and GLM-5.1 at
their native selection granularities (Section~\ref{sec:analysis}). A consumer can therefore
retain many of its preferred candidates without repeating the full scan, provided that it
can rank the candidates for itself.

We introduce \textbf{MLIPS (Multi-Layer Indexer Pool Sharing)}. A producer runs the full
indexer, selects its own top-$K$, and publishes a wider top-$M$ pool. Each consumer scores
those candidates with its own indexer and selects its own top-$K$. Consumers preserve their
attention computation and the checkpoint's organization of KV state. The restriction is
in the search domain: a consumer can re-rank pooled entries, but cannot recover an omitted
one. Layerwise sharing applies during prefill and decode. Persistent decode reuse extends
a pool across later generation steps; a coarse producer further reduces candidate-generation
work by scoring groups of keys.

We implement these paths in vLLM and evaluate DeepSeek-V4-Flash on four long-context
suites. With three consumer layers per producer and no reuse across decode steps,
MLIPS gives $1.14$--$1.40\times$ speedups and $14.8$--$27.4\%$ lower GPU energy.
Observed accuracy changes range from $-1.49$ to $+0.29$ points
(Table~\ref{tab:main-results}).

The traces also expose a broader opportunity: nearby queries and sparse layers repeatedly
select a shared working set of KV positions. We characterize this structure and suggest
Top-$G$, a bounded pool assembled from a query group's producer selections
(Section~\ref{sec:topg}). This is an empirical observation and an exploratory construction;
turning it into a serving mechanism remains future work.

\begin{samepage}
Our contributions are:
\begin{itemize}
\item \textbf{Evidence for candidate reuse.} We distinguish final-selection overlap from
broad-pool coverage across three model families and characterize its dependence on query
position and sharing extent.
\item \textbf{An inference-time sharing method and serving evaluation.} We develop MLIPS's
producer--consumer design, implement layerwise and persistent decode reuse with a coarse
producer option, and measure accuracy, wall time, and GPU energy.
\item \textbf{Shared working-set structure in prefill.} We measure recurring selected KV
positions across queries and layers and outline Top-$G$ as a way to exploit that structure.
\end{itemize}
\end{samepage}

% Source fragment: sections/background_related_work.tex
\section{Background and Related Work}
\label{sec:background-related-work}
\label{sec:background}

\paragraph{Sparse attention and candidate search.}
An indexer assigns relevance scores to causally visible keys and selects the top-$K$ indices
for attention \citep{deepseekv32}. The indexer scores and the attention weights are separate:
attention computes its own weights after selection. We denote the visible search domain by
$N$ entries. In DeepSeek-V4-Flash's compressed sparse-attention (CSA) layers, these are
compressed entries, with approximately $N=L/4$ for context length $L$, and $K=512$
\citep{deepseekv4}. A separate local-attention window remains active.
A full indexer scans $N$ entries even though the global attention branch uses only $K$.
This creates the opportunity to share candidate search.

\paragraph{Pools, selections, and schedules.}
A producer publishes $M$ candidate indices, where $K<M\ll N$; a consumer scores those
candidates and selects its own top-$K$. Selection inheritance instead directly adopts the
producer's top-$K$. Pools contain indices, not shared attention scores or KV vectors.
We use $h$ for the number of consumer sparse layers served by each producer and $qh$
for the number of subsequent decode steps served before a rebuild. Exact
producer--consumer distance is denoted $d$ when discussing an individual layer pair;
sparse-layer distance need not equal Transformer-layer distance. Schedule averages at
different $h$ can include different producer placements and several distances.

\paragraph{What coverage measures.}
Let $\mathcal{C}_K$ be a consumer's top-$K$ under full search and let $\mathcal{P}_K$
and $\mathcal{P}_M$ be the producer's selection and wider pool. We measure
\begin{equation}
R_{KK}=\frac{|\mathcal{C}_K\cap\mathcal{P}_K|}{K},
\qquad
R_{KM}=\frac{|\mathcal{C}_K\cap\mathcal{P}_M|}{K}.
\label{eq:recalls}
\end{equation}
Selection overlap $R_{KK}$ describes agreement between final selections. Pool coverage
$R_{KM}$ describes the consumer's preferred entries available for re-selection.
Nested producer sets imply $R_{KM}\geq R_{KK}$. Neither metric measures retained
attention mass or downstream accuracy. Across model families, counts use native selection
units---compressed entries, blocks, or tokens---so equal pool sizes need not have equal
cost.

\paragraph{Cross-layer reuse.}
\label{sec:related}
IndexCache studies top-$K$ inheritance, using calibration to choose retained indexers
without weight updates or distillation to train them \citep{indexcache}. IndexShare and
LongCat Sparse Attention also share selections through trained architectures
\citep{glm52,lsa}. DeepSeek-V4.1-Flash is the closest pool-based design: during decode,
a full indexer constructs a wider candidate pool and deeper Reindex layers score that pool
independently; the restriction is introduced during post-training \citep{deepseekv41}.
MLIPS evaluates pool re-selection on existing checkpoints and extends reuse across
successive decode steps. It preserves the host checkpoint's attention-state organization.

\paragraph{Coarse search and query sharing.}
HISA scores mean-pooled block representatives before fine selection \citep{hisa}.
MLIPS's coarse producer combines this type of filtering with pool sharing across layers.
PIVOT forms a proxy from jointly available queries and runs one shared prefix scan;
its Refine variant re-scores a broader candidate set for each query \citep{pivot}.
This establishes query-group candidate search. Our persistent decode pool instead serves
queries that arrive at later steps. The exploratory Top-$G$ study examines producer-selection
memberships across query groups and layers, rather than constructing a proxy query.

\paragraph{Routing and KV storage.}
Cross-Layer Attention and YOCO share KV state across layers
\citep{brandon2024cla,sun2024yoco}; You Only Index Once combines shared routing with
YOCO \citep{clsa}. MLIPS shares candidate indices without requiring KV sharing.
Restricting the active working set may support offloading, but does not itself reduce
total KV storage. We analyze that opportunity separately in Appendix~\ref{app:offload}.
Together, these distinctions position MLIPS as a study of training-free candidate reuse
and its measured quality--efficiency tradeoffs.

% Source fragment: sections/analysis.tex
\section{Evidence for Candidate Reuse}
\label{sec:analysis}

% Source fragment: figures/fig_reuse_summary.tex
\begin{figure*}[t]
\centering
\begin{minipage}[t]{0.49\textwidth}
\centering
\begin{tikzpicture}
\begin{axis}[mlipsaxis,width=\linewidth,height=4.3cm,
  ybar=3pt,bar width=12pt,ymin=0,ymax=1,
  symbolic x coords={LB2,RULER,MRCR},xtick=data,
  ylabel={Index recall},ytick={0,.25,.5,.75,1},
  nodes near coords,nodes near coords style={font=\scriptsize},
  legend style={at={(0.5,1.05)},anchor=south,legend columns=2,
    draw=none,font=\scriptsize},
  enlarge x limits=0.25]
\addplot[fill=mlipsorange!65,draw=mlipsorange] coordinates
  {(LB2,.34) (RULER,.23) (MRCR,.28)};
\addplot[fill=mlipsblue!60,draw=mlipsblue] coordinates
  {(LB2,.85) (RULER,.72) (MRCR,.81)};
\legend{Selection overlap $R_{KK}$,Pool coverage $R_{KM}$}
\end{axis}
\end{tikzpicture}
\par{\small (a) All sampled query positions in up-to-1M captures.}
\end{minipage}\hfill
\begin{minipage}[t]{0.49\textwidth}
\centering
\begin{tikzpicture}
\begin{axis}[mlipsaxis,width=\linewidth,height=4.3cm,
  ymin=0,ymax=1,xmin=-.2,xmax=2.2,
  xtick={0,1,2},xticklabels={0--200K,200--500K,500K--1M},
  ylabel={Index recall},ytick={0,.25,.5,.75,1},
  nodes near coords,nodes near coords style={font=\scriptsize},
  legend style={at={(0.5,1.05)},anchor=south,legend columns=2,
    draw=none,font=\scriptsize}]
\addplot[color=mlipsorange,mark=square*,thick] coordinates
  {(0,.31) (1,.23) (2,.20)};
\addplot[color=mlipsblue,mark=*,thick] coordinates
  {(0,.90) (1,.76) (2,.66)};
\legend{$R_{KK}$,$R_{KM}$}
\end{axis}
\end{tikzpicture}
\par{\small (b) RULER-1M, grouped by query position.}
\end{minipage}
\caption{A wider pool increases available consumer choices, but coverage declines
at later positions. DeepSeek-V4-Flash, $h=1$, $K=512$, $M=16{,}384$;
values are rounded from the capture summaries. Bands refer to positions
within prompts. Full cross-model results appear in Appendix~\ref{app:recall}.}
\label{fig:reuse-summary}
\end{figure*}

We first ask whether existing indexers provide enough shared candidate structure to
justify restricting a consumer's search. We measure the index recalls in
Equation~\ref{eq:recalls} on DeepSeek-V4-Flash, MiniMax-M3, and GLM-5.1 captures
\citep{deepseekv4,msa,glm51},
using each family's native selection unit. Captures use the ten largest prompts from each
of three suites. A context band denotes \emph{query positions within prompts}, rather than
a separate prompt-length condition. Full heatmaps, capture details, and pooling conventions
are in Appendix~\ref{app:recall}; Figure~\ref{fig:reuse-summary} highlights the main trends.

\paragraph{Broad pools retain more consumer choices.}
We use the consumer's independently computed full-search top-$K$ as the reference.
Pool coverage measures the fraction of these indices present in the producer's wider
top-$M$ pool; final-selection overlap counts only those also in the producer's top-$K$.
At query positions below $200$K, broad pools retain $89$--$94\%$ of the consumer's
choices for adjacent producer--consumer pairs across the three families and suites.
The producer's final top-$K$ retains only $31$--$65\%$.
The pools contain $16{,}384$ compressed entries for DeepSeek-V4-Flash, $128$ blocks
for MiniMax-M3, and $16{,}384$ tokens for GLM-5.1.

Across query positions in up-to-$1$M captures, DeepSeek-V4-Flash's pool coverage is
$72$--$85\%$, compared with $23$--$34\%$ final-selection overlap. MiniMax-M3 gives
$74$--$85\%$ coverage, compared with $47$--$56\%$ overlap.
Thus a wider pool makes many more of the consumer's preferred indices available for
re-selection. This measures candidate availability, not attention-weight agreement
or task accuracy; task-level consequences are evaluated in Section~\ref{sec:evaluation}.

\paragraph{Sharing has a context and depth boundary.}
Coverage declines toward later query positions. On RULER-1M, DeepSeek-V4-Flash's
adjacent-layer pool coverage falls from $0.90$ to $0.76$ to $0.66$ across the
$0$--$200$K, $200$--$500$K, and $500$K--$1$M position bands.
Serving more consumer layers from one producer also generally lowers coverage, because
the later consumers are farther from the producer. However, each schedule averages
over a different set of layer pairs. With one consumer per producer ($h=1$), for example,
the third sparse layer is a producer; with two ($h=2$), it consumes the first layer's
pool. Such changes can produce local increases in the average even as sharing grows.
Pool width and sharing extent should therefore be evaluated together, checking
individual layer pairs as well as the schedule average.

\paragraph{Nearby queries share candidates too.}
We also hold a pool fixed and measure how many full-search choices it contains for
later queries in the same layer. Across query positions in the RULER-1M captures,
DeepSeek-V4-Flash retains $84.6\%$ of these choices one query later and $78.5\%$ three
queries later. MiniMax-M3 retains $91.8\%$ and $89.7\%$ at the same offsets.
For comparison, adjacent-layer reuse at the same query gives $72.5\%$ coverage on
DeepSeek-V4-Flash and $74.2\%$ on MiniMax-M3.

These measurements come from sampled bursts of four consecutive prompt positions
during prefill. They test reuse at offsets of one to three positions within each burst,
not longer intervals or model-generated continuations. The results motivate short-range
temporal reuse, but do not show how a persistent pool affects generation quality;
we test that separately in Section~\ref{sec:eval-2d}. A stale pool can miss both new
keys and older keys that become relevant.

These findings support a method that shares broad candidates while preserving fresh
consumer scoring. The serving evaluation tests whether the resulting search restriction
pays off in practice (Section~\ref{sec:evaluation}). Section~\ref{sec:topg} then asks how
far the shared selected-index structure extends across larger prefill query groups.

% Source fragment: sections/methodology.tex
\section{MLIPS}
\label{sec:methodology}

% Source fragment: figures/fig_overview.tex
\begin{figure*}[t]
\centering
\begin{minipage}[t]{0.48\textwidth}
\centering
\begin{tikzpicture}[x=1cm,y=1cm,>=Stealth,
  box/.style={draw=mlipsblue!70,rounded corners=2pt,fill=mlipsblue!7,
  align=center,font=\small,inner sep=5pt},
  arr/.style={->,thick,draw=black!65}]
\node[box,fill=mlipsorange!15,draw=mlipsorange] (p) at (0,1.9)
  {Producer: score $N$ keys\\own top-$K$ + candidate pool};
\node[box] (pool) at (0,0.8) {$M$ shared candidate indices};
\node[box,text width=2.2cm] (c1) at (-1.45,-0.4)
  {Consumer $\ell$\\own scores\\top-$K$, own KV};
\node[box,text width=2.2cm] (c2) at (1.45,-0.4)
  {Consumer $\ell+1$\\own scores\\top-$K$, own KV};
\draw[arr] (p) -- (pool);
\draw[arr] (pool) -- (c1);
\draw[arr] (pool) -- (c2);
\end{tikzpicture}
\par\smallskip{\small (a) Share the search domain; re-select independently.}
\end{minipage}\hfill
\begin{minipage}[t]{0.49\textwidth}
\centering
\begin{tikzpicture}[x=0.69cm,y=0.61cm,
  p/.style={draw=mlipsorange,fill=mlipsorange!35,minimum width=5.8mm,
    minimum height=4.8mm,font=\scriptsize,inner sep=0pt},
  c/.style={draw=mlipsblue!65,fill=mlipsblue!12,minimum width=5.8mm,
    minimum height=4.8mm,font=\scriptsize,inner sep=0pt}]
\foreach \x in {0,1,2}{
  \node[font=\scriptsize] at (\x,1) {$\ell_{\x}$};
  \node[font=\scriptsize] at (\x+4.2,1) {$\ell_{\x}$};
}
\foreach \y/\q in {0/0,-1/1,-2/2}{
  \node[font=\scriptsize,anchor=east] at (-0.5,\y) {$q+\q$};
  \node[p] at (0,\y) {P};
  \node[c] at (1,\y) {C};
  \node[c] at (2,\y) {C};
  \foreach \x in {0,1,2}{\node[c] at (\x+4.2,\y) {C};}
}
\node[p] at (4.2,0) {P};
\foreach \y in {0,-1,-2}{
  \draw[dashed,rounded corners,draw=mlipsblue!80] (-.45,\y+.42)
    rectangle (2.45,\y-.42);
}
\draw[dashed,rounded corners,draw=mlipsblue!80] (3.75,.42)
  rectangle (6.65,-2.42);
\node[font=\small] at (1,-3.1) {Layer sharing};
\node[font=\small] at (5.2,-3.1) {Persistent decode reuse};
\end{tikzpicture}
\par\smallskip{\small (b) One pool serves a row or a tile of queries and layers.}
\end{minipage}
\caption{MLIPS producer--consumer design. Orange cells build a pool; blue cells
re-score its candidates. Dashed boxes show reuse at $h=2$ and, on the right,
$qh=2$. Consumers retain independent selections, which may overlap.}
\label{fig:overview}
\end{figure*}

\subsection{Share candidates, retain independent selection}
\label{sec:principle}
A producer runs the checkpoint's indexer over the visible domain, keeps its own top-$K$,
and publishes the top-$M$ candidate indices. Each consumer gathers its own indexer keys
at those positions, computes fresh scores, and selects its own top-$K$ for attention.
Only candidate addresses are shared: the consumer uses its own query, indexer keys,
attention weights, and KV entries. Figure~\ref{fig:overview} summarizes the mechanism.

The approximation is the restricted search domain. Re-scoring preserves the consumer's
ranking within the pool, but cannot recover an omitted key. For nested producer sets,
$\mathcal{P}_K\subseteq\mathcal{P}_M$ implies $R_{KM}\geq R_{KK}$; this is a coverage
property, not a guarantee of equal task accuracy. Appendix~\ref{app:cost} gives the
formal argument and operation accounting.

\subsection{Layerwise and persistent decode reuse}
\label{sec:mlips1d}
\paragraph{Layer sharing (1D MLIPS).}
We partition the $n$ indexer-equipped sparse layers into groups of at most $h+1$. The first layer in
each group produces a pool; the following $h$ consumers re-select from it.
The producer count is $p=\lceil n/(h+1)\rceil$, and the consumer count is $c=n-p$.
The final group may be shorter. Pools are built independently for each query, so layer
sharing applies during both prefill and decode. Changing $h$ changes producer placement
as well as the maximum consumer distance; it is not simply a continuous increase in
one fixed pair's distance.

\paragraph{Persistent decode pools (2D MLIPS).}
\label{sec:mlips2d}
A producer can retain its pool for $qh$ subsequent decode steps. It rebuilds once every
$qh+1$ steps and re-scores the stored candidates in between; consumer layers also
re-score the pool at each step. Each later query therefore receives fresh scores,
although the candidate membership is older. This persistence serves queries as they
arrive, extending the reuse schedule along the query axis.

A stored pool can omit older keys that become relevant and compressed keys created
after the build. The local-attention window still provides access to recent tokens,
but does not guarantee that restricted global selection matches full search.
Our measurements use $qh\leq3$ and compression ratio $r=4$. Admitting newly materialized
keys and triggering adaptive refreshes are possible extensions, not evaluated
corrections to these results.

\subsection{Indexer scoring cost}
\label{sec:split}
For $Q$ queries and $R$ candidates, the dominant dot-product work per sequence is
\begin{equation}
I(Q,R)=2H_iD_iQR,
\label{eq:indexer-ops}
\end{equation}
where $H_i$ and $D_i$ are the indexer head count and dimension. Summing over the
$n$ indexer-equipped layers, the full-search baseline costs $nI(Q,N)$.
Layer sharing replaces this with
\begin{equation}
C_{\mathrm{1D}}=pI(Q,N)+cI(Q,M).
\label{eq:split}
\end{equation}
Averaging over one decode rebuild cycle gives the scoring-work reduction
\begin{equation}
R_{\mathrm{ops}}(h,qh)=
\frac{nN(qh+1)}{pN+\left[p\,qh+c(qh+1)\right]M}.
\label{eq:2d-exact}
\end{equation}
For a concrete example, DeepSeek-V4-Flash has $n=21$ indexer-equipped CSA layers;
this counts only the layers that perform indexed sparse selection, not all backbone
layers. Its indexer uses $H_i=64$, $D_i=128$, and compression ratio $r=4$.
At $L=2^{20}$ tokens, the compressed domain therefore contains
$N=L/4=262{,}144$ entries. With $M=16{,}384$, Equation~\ref{eq:2d-exact} gives
scoring-work reductions of $3.03\times$ for $(h,qh)=(3,0)$ and $3.25\times$ for $(1,1)$.
These ratios describe indexer dot products. They exclude top-$K$/top-$M$ selection,
gathering, other score arithmetic, and memory traffic, and do not predict proportional
end-to-end speedups.

\begin{samepage}
\subsection{Coarse candidate generation}
\label{sec:mlipscp}
A \emph{coarse producer} first searches groups of keys, then refines its selection
within the retained groups. We call this variant MLIPS with a Coarse Producer
(MLIPS-CP).\par
\end{samepage}

It reduces the producer scan by grouping $P$ consecutive compressed keys
and scoring one mean representative per group. The producer keeps $b$ groups with
$bP=M$, publishes their members as the candidate pool, and performs a fine selection
for its own top-$K$. Consumers perform their usual re-selection within that pool. This combines
coarse-to-fine indexing~\citep{hisa} with cross-layer candidate sharing. Its modeled
reduction is
\begin{equation}
R_{\mathrm{CP}}(P,h,qh)=
\frac{nN(qh+1)}{pN/P+n(qh+1)M}.
\label{eq:cp-exact}
\end{equation}
The main coarse configuration uses $P=2$ and $b=8{,}192$.
Pooling can filter out useful groups; consumer refinement cannot restore them.
We evaluate this additional restriction separately from exact-producer sharing.
Appendix~\ref{app:cost} contains the parameter grids and an explanation of
representative-score dilution.

Appendix~\ref{app:implementation} describes the vLLM integration, buffers,
CUDA-graph handling, and selection checks.

% Source fragment: sections/evaluation.tex
\section{Evaluation}
\label{sec:evaluation}

\subsection{Setup}

We evaluate four long-context suites: LongBench-v2 ($503$ items), RULER-1M
($130$ items, ten per task), MRCR-v2 8-needle ($784$ items), and GraphWalks
($1{,}150$ items) \citep{longbenchv2,ruler,mrcr,graphwalks}.
The serving stack uses DeepSeek-V4-Flash in vLLM on one NVIDIA GB300 node,
with four GPUs at
$1{,}400$~W each and tensor parallelism $4$. The pool width is
$M=16{,}384$. Chunked prefill uses an $8{,}192$-token budget, with at most
$24$ concurrent sequences and $24$ requests in flight; prefix caching is
off. Expert parallelism and speculative decoding are disabled, and CUDA
graphs are enabled. Stock and MLIPS use the same checkpoint and software;
configuration flags select the indexer path
(Appendix~\ref{app:implementation}). Appendix~\ref{app:setup} gives precision,
software, and measurement details.

We report each suite's task score on a $0$--$100$ scale, end-to-end
workload speedup relative to stock, and GPU energy saved. Energy integrates
four-GPU power samples taken every $500$~ms, from the end of warm-up to
workload completion; it excludes startup, compilation, and graph capture,
as well as CPU, host-memory, and network energy. Every configuration is a
single serving run. Two stock boots differed by up to $1.3\%$ in time and
$0.3\%$ in energy, with score differences of approximately one to two
points or less across suites. This limited comparison does not establish
accuracy equivalence or statistical significance. We report observed
differences and retain each experiment campaign's own baseline.

\subsection{Layer sharing: consistent serving gains}
\label{sec:eval-1d}

Table~\ref{tab:main-results} evaluates $h=1,2,3$ with no persistent decode
reuse ($qh=0$). Every configuration is faster and uses less GPU energy than
its stock run. At $h=3$, speedups range from $1.139\times$ on MRCR-v2 to
$1.398\times$ on RULER-1M, with energy savings of $14.8$--$27.4\%$.
The corresponding score changes, in table order, are $-0.40$, $-0.29$,
$-1.49$, and $+0.29$ points. These observations support sharing a broad
candidate pool while preserving each consumer's independent selection;
the search restriction still changes some task scores.

\begin{table*}[t]
\centering\small
\setlength{\tabcolsep}{7pt}
\caption{Layer sharing with $M=16{,}384$ and $qh=0$.
Scores, speedups, and GPU energy savings use each suite's stock run.
Every row is one serving run.}
\label{tab:main-results}
\begin{tabular}{llrrr}
\toprule
Suite & $h$ & Score & Speedup & GPU energy saved \\
\midrule
LongBench-v2 & stock & $50.50$ & $1.000\times$ & $0.0\%$ \\
 & $1$ & $49.11$ & $1.066\times$ & $8.9\%$ \\
 & $2$ & $49.70$ & $1.162\times$ & $15.3\%$ \\
 & $3$ & $50.10$ & $1.197\times$ & $17.1\%$ \\
\midrule
RULER-1M & stock & $80.41$ & $1.000\times$ & $0.0\%$ \\
 & $1$ & $79.91$ & $1.177\times$ & $16.9\%$ \\
 & $2$ & $78.46$ & $1.342\times$ & $25.1\%$ \\
 & $3$ & $80.12$ & $1.398\times$ & $27.4\%$ \\
\midrule
MRCR-v2 & stock & $75.02$ & $1.000\times$ & $0.0\%$ \\
 & $1$ & $73.59$ & $1.015\times$ & $6.6\%$ \\
 & $2$ & $74.64$ & $1.117\times$ & $13.2\%$ \\
 & $3$ & $73.53$ & $1.139\times$ & $14.8\%$ \\
\midrule
GraphWalks & stock & $53.35$ & $1.000\times$ & $0.0\%$ \\
 & $1$ & $54.50$ & $1.131\times$ & $13.9\%$ \\
 & $2$ & $52.56$ & $1.290\times$ & $21.9\%$ \\
 & $3$ & $53.64$ & $1.326\times$ & $23.7\%$ \\
\bottomrule
\end{tabular}
\end{table*}

The size of the gain depends on how much work reaches the sharing path.
All RULER-1M prompts exceed its roughly $65$K-token activation threshold,
whereas about half of MRCR-v2 prompts do not. The speedup differences are
consistent with that distinction, although the aggregate measurements do
not isolate indexer time from other serving costs. Measured speedups are
therefore the relevant performance result; the operation model in
Section~\ref{sec:methodology} explains the opportunity rather than
predicting end-to-end gains directly.

Quality is not monotone in $h$. RULER-1M loses $1.95$ points at $h=2$, but
only $0.29$ at $h=3$. GraphWalks also has a reported long-context BFS dip at
$h=2$ (Appendix~\ref{app:results}); its cause remains unresolved.
Extending sharing to $h=5,6$ increases speed further, reaching
$1.572\times$ on RULER-1M at $h=6$, but its score then falls by $4.42$
points. The other three suites have smaller aggregate changes in these
runs. Thus the evidence favors modest sharing as the initial operating
region, while deeper schedules require workload-specific quality checks.
Appendix~\ref{app:results} retains all extended configurations.

Energy and latency also need separate measurements. At $h=1$, mean GPU
power falls by $2.2$--$5.2\%$, so energy savings exceed the fraction of wall
time saved. At deeper schedules, mean power can rise even as the run
finishes sooner. Lower indexer work consequently does not imply a fixed
power reduction.

\subsection{Decode reuse: a modest gain with quality risks}
\label{sec:eval-2d}

Persistent reuse across decode steps
(Section~\ref{sec:mlips2d}) adds little performance benefit on the four
long-context suites. With $h=0$, all $qh=1,2,3$ runs achieve only
$0.960$--$1.000\times$ stock speed; adding reuse to layer sharing changes
time little while sometimes reducing task scores substantially.
At $(h,qh)=(3,2)$, score changes are $-1.79$, $-9.68$, $-5.29$, and
$-3.01$ points on LongBench-v2, RULER-1M, MRCR-v2, and GraphWalks,
respectively. The full grid appears in Table~\ref{tab:eval-2d} and its
continuations. These results show that query locality alone does not
justify a fixed refresh schedule for all workloads.

To isolate decode, we use $30$ synthetic prompts of approximately $0.9$M
tokens, offering $24$ sessions with prefix caching disabled. We subtract
the wall time of a one-token generation from that of a $40{,}000$-token
generation. Table~\ref{tab:eval-decode} reports the mean of two draws after
discarding a warm-up draw, with one boot per arm and $h=0$ throughout.

\begin{table}[t]
\centering\small
\caption{Isolated decode at $h=0$. Changes are relative to stock.
This synthetic timing experiment does not measure answer quality.}
\label{tab:eval-decode}
\begin{tabular}{lrr}
\toprule
$qh$ & Decode time (s) & Change \\
\midrule
$0$ (stock) & $857.2$ & --- \\
$1$ & $867.2$ & $+1.2\%$ \\
$2$ & $850.5$ & $-0.8\%$ \\
$3$ & $829.0$ & $-3.3\%$ \\
\bottomrule
\end{tabular}
\end{table}

At $qh=3$, reuse reduces the measured decode time by $3.3\%$; at $qh=1$,
it increases time by $1.2\%$. This supports a limited performance
opportunity for longer generations, but neither combined layer-and-query
speedup nor generated-answer quality is established by this instrument.
A practical refresh policy must address newly relevant existing keys as
well as keys appended since the pool was built.

\subsection{Coarse producers: energy savings and failure cases}
\label{sec:eval-cp}

% Source fragment: figures/fig_tradeoffs.tex
\begin{figure*}[t]
\centering
\begin{minipage}[t]{0.49\textwidth}
\centering
\begin{tikzpicture}
\begin{axis}[mlipsaxis,width=\linewidth,height=4.1cm,
  xlabel={Layer-sharing schedule $h$},ylabel={Accuracy change (points)},
  xtick={1,2,3,5,6},ymin=-5,ymax=2,ytick={-4,-2,0,2},
  legend style={at={(0.02,0.03)},anchor=south west,draw=none,
    fill=white,font=\scriptsize,legend columns=2}]
\addplot[gray,dashed,forget plot] coordinates {(1,0) (6,0)};
\addplot[mlipsblue,mark=*,thick] coordinates
 {(1,-1.39) (2,-.80) (3,-.40) (5,-.20) (6,-1.39)};
\addplot[mlipsorange,mark=square*,thick] coordinates
 {(1,-.50) (2,-1.95) (3,-.29) (5,-3.33) (6,-4.42)};
\addplot[mlipsgreen,mark=triangle*,thick] coordinates
 {(1,-1.43) (2,-.38) (3,-1.49) (5,-.97) (6,-.12)};
\addplot[mlipspurple,mark=diamond*,thick] coordinates
 {(1,1.15) (2,-.79) (3,.29) (5,-1.67) (6,-.28)};
\legend{LB2,RULER,MRCR,GraphWalks}
\end{axis}
\end{tikzpicture}
\par{\small (a) Exact producer, $qh=0$.}
\end{minipage}\hfill
\begin{minipage}[t]{0.49\textwidth}
\centering
\begin{tikzpicture}
\begin{axis}[mlipsaxis,width=\linewidth,height=4.1cm,
  xlabel={Coarse group size $P$},ylabel={Accuracy change (points)},
  xmode=log,log basis x=2,xtick={2,4,8,16},
  xticklabels={2,4,8,16},ymin=-16,ymax=3,ytick={-15,-10,-5,0},
  legend style={at={(0.02,0.03)},anchor=south west,draw=none,
    fill=white,font=\scriptsize,legend columns=2}]
\addplot[gray,dashed,forget plot] coordinates {(2,0) (16,0)};
\addplot[mlipsblue,mark=*,thick] coordinates
 {(2,-.60) (4,1.79) (8,-.80) (16,.39)};
\addplot[mlipsorange,mark=square*,thick] coordinates
 {(2,.28) (4,-4.08) (8,-9.01) (16,-14.30)};
\addplot[mlipsgreen,mark=triangle*,thick] coordinates
 {(2,-.75) (4,-.04) (8,.17) (16,-3.47)};
\addplot[mlipspurple,mark=diamond*,thick] coordinates
 {(2,-.01) (4,1.25) (8,-1.36) (16,-2.52)};
\legend{LB2,RULER,MRCR,GraphWalks}
\end{axis}
\end{tikzpicture}
\par{\small (b) Coarse producer, $h=1$, $qh=0$, fixed $M$.}
\end{minipage}
\caption{Observed quality limits of more extensive sharing and coarse filtering.
Each configuration is one serving run; changes use the corresponding campaign's
stock baseline. Curves connect measured settings and are not uncertainty estimates.
Full timing, energy, and accuracy values are in Appendix~\ref{app:results}.}
\label{fig:tradeoffs}
\end{figure*}

The coarse producer replaces the dense producer scan with group
representatives (Section~\ref{sec:mlipscp}). At $P=2$, $b=8{,}192$,
$h=1$, and $qh=0$, its four speedups are $1.027$, $1.229$, $0.979$, and
$1.087\times$, with GPU energy savings of $10.2$, $22.5$, $7.7$, and
$16.5\%$, respectively. Score changes against this campaign's baseline are
$-0.60$, $+0.28$, $-0.75$, and $-0.01$ points. Energy falls even on
MRCR-v2, where latency worsens. Relative to the separately measured exact
producer, the coarse path is not consistently faster.

More aggressive settings expose substantial failures. At $P=2$, RULER-1M
loses $5.71$ points at $h=2$, and MRCR-v2 loses $3.30$ at $h=3$.
Increasing $P$ from $2$ to $16$ at fixed $h=1$ and $bP=M$ raises energy
savings on every suite, but the $P=16$ score changes are $+0.39$,
$-14.30$, $-3.47$, and $-2.52$ points. Replacing mean representatives by
elementwise maxima does not remove the RULER-1M failure.
Tables~\ref{tab:eval-cp}--\ref{tab:eval-rep} give all settings and raw
measurements. Figure~\ref{fig:tradeoffs} shows the quality changes for deeper
layer sharing and larger coarse groups.

Coarse-producer results use a separate campaign baseline and one run per
configuration. Attributing the observed losses to representative dilution
requires verified matched controls and repeats. Exact layer sharing is the
clearest measured benefit; decode reuse and coarse production expose
additional workload-dependent tradeoffs.

% Source fragment: sections/top_G.tex
\section{Shared Candidate Working Sets}
\label{sec:topg}

MLIPS shares a candidate pool across layers for each query. The traces also
suggest that nearby queries repeatedly select a common set of KV positions.
This raises a broader question: can a bounded pool represent the working set
of a group of queries across several sparse layers? We investigate this
structure through an exploratory \emph{Top-$G$} construction. The evidence
concerns shared selected positions; it does not establish similarity among
the corresponding KV vectors.

Let $\ell$ index CSA layers, and let $K_q^\ell$ and $M_q^\ell$ be the
producer's Top-$K$ and Top-$M$ index sets for query $q$. For a group of
recorded queries $C$, we construct one pool
\begin{equation}
 \mathcal{G}_{\ell,C}\subseteq
 \bigcup_{q\in C}(K_q^\ell\cup M_q^\ell),
 \qquad |\mathcal{G}_{\ell,C}|\leq G.
 \label{eq:topg-pool}
\end{equation}
The constructor prioritizes producer Top-$K$ memberships from queries whose
selections remain poorly covered, using residual Top-$M$ memberships to
break ties. This provides a concrete probe of shared structure without
using selections from consumer layers. Appendix~\ref{app:topg} gives the
algorithm and sampling protocol. We freeze the pool and measure
\begin{equation}
 R_{\ell,q}(d)=
 \frac{|K_q^{\ell+d}\cap\mathcal{G}_{\ell,C}|}
      {|K_q^{\ell+d}|},\qquad d\geq1,
 \label{eq:topg-recall}
\end{equation}
where $d$ counts subsequent CSA layers, separated by two physical model
layers in these traces.

\begin{figure}[t]
 \centering
 \definecolor{topgblue}{RGB}{18,91,139}
 \resizebox{\columnwidth}{!}{%
 \begin{tikzpicture}[x=.93cm,y=.46cm,font=\footnotesize]
   \node[anchor=east,align=right] at (-.55,.82) {Nominal\\query window};
   \node at (3,.99) {Subsequent CSA-layer distance};
   \foreach \c/\lab in {0/1,1/2,2/3,3/4,4/5,5/6,6/Mean}{
     \node at (\c,.36) {\lab};
   }
   \foreach \r/\lab in {0/{4,096},1/{16,384},2/{65,536},3/{Sequence}}{
     \node[anchor=east] at (-.55,-\r-.4) {\lab};
   }
   \foreach \r/\vals in {
     0/{82.4,79.4,79.4,76.9,74.3,73.5,77.6},
     1/{74.1,71.1,71.2,68.6,66.0,65.5,69.4},
     2/{57.1,54.1,54.6,52.0,49.5,49.8,52.9},
     3/{36.5,34.5,35.5,32.7,30.6,31.7,33.6}}{
     \foreach \val [count=\c from 0] in \vals {
       \fill[topgblue!\val!white]
         (\c-.49,-\r-.88) rectangle (\c+.49,-\r+.08);
       \ifdim\val pt>64pt
         \node[text=white] at (\c,-\r-.4) {\val};
       \else
         \node at (\c,-\r-.4) {\val};
       \fi
     }
   }
   \draw[white,line width=1.2pt] (5.5,.08)--(5.5,-3.88);
   \node[anchor=west] at (-.49,-4.45) {Top-$K$ index recall (\%); $G=32{,}000$};
 \end{tikzpicture}}
 \caption{A shared pool covers many subsequent selections, with coverage
 decreasing as its query window grows. Results use sampled queries from
 LB2 sequences 0--6; window sizes describe original positions, not the
 number of measured queries. Mean averages distances 1--6.}
 \label{fig:topg-recall}
\end{figure}

The capture contains eight LB2 sequences and 21 CSA layers. The primary
results cover sequences 0--6, with $K=512$, nominal $M=16\mathrm{K}$,
$G=32{,}000$, and context length at least $G$. Crucially, the trace records
four queries per 256 original positions: a nominal 4,096-position window
contains 64 recorded queries, not 4,096 measured queries.

Figure~\ref{fig:topg-recall} shows a shared working set across both axes.
For the smallest window, recall decreases from $82.4\%$ at the next CSA
layer to $73.5\%$ six layers later, averaging $77.6\%$. Increasing the
nominal window to 16,384 or 65,536 positions reduces mean recall to
$69.4\%$ or $52.9\%$; sequence-wide sharing reaches $33.6\%$.
Thus a fixed pool captures recurring selections but cannot absorb all
query-specific variation. Per-sequence measurements show the same window
ordering (Appendix~\ref{app:topg}).

These observations motivate consolidating producer selections into a
shared Top-$G$ pool, with each consumer scoring and selecting independently.
Shared index metadata and coordinated KV staging are possible uses. Their
benefits require measurement: $G>M$ can increase consumer scoring, and
index coverage alone establishes neither task accuracy nor speedup.
Construction over a whole query group can also depend on later queries;
a serving design must specify information availability and causal masks.
The present result is evidence for exploitable structure and a proposed
way to represent it. It leaves deployment and full-query validation open.

% Source fragment: sections/limitations.tex
\section{Limitations and Future Work}
\label{sec:limitations}

\paragraph{Evidence and generalization.}
Serving results cover one checkpoint and hardware stack; the cross-model evidence concerns
index coverage. Each serving configuration is a single run. Two stock boots provide a
limited variability reference, not an equivalence test or uncertainty estimate for every
comparison. Repeated paired evaluations and matched implementations of the closest reuse
baselines are needed to establish robustness and comparative benefit. The operation model
also excludes selection, gathering, and memory costs; phase-level profiling would clarify
why scoring reductions translate unevenly into latency savings.

\paragraph{Adaptive reuse.}
Uniform schedules expose weak layer pairs and unexplained quality changes, including the
$h=2$ dip. A useful next step is to select producer locations and pool widths using
calibration, then evaluate the chosen schedule on held-out workloads. Persistent decode
reuse needs refresh rules that balance pool age, new-key admission, and shifts in query
relevance. Generated-token fidelity and task accuracy should accompany timing results;
prefill locality alone cannot settle those questions.

\paragraph{From shared structure to Top-$G$.}
The prefill traces identify recurring selected positions, without establishing KV-vector
similarity or a deployed compression scheme. Testing all query rows, additional prompts,
and alternative constructors would show how robust the apparent working sets are.
A serving design must account for pool-construction cost, causal information flow, per-query
masking, task accuracy, and the cost of retrieving pooled entries. The current sampled
recall measurements do not establish a serving speedup.

\paragraph{Memory management.}
Candidate pools may bound a consumer's immediate working set, but realizing lower GPU
residency requires cache management and transfers. The offloading analysis in
Appendix~\ref{app:offload} is therefore a capacity opportunity rather than a measured
memory reduction. Future evaluation should report peak resident memory and transfer
overhead together with latency and quality.

% Source fragment: sections/conclusion.tex
\balance
\section{Conclusion}
\label{sec:conclusion}

MLIPS shares candidate search across sparse-attention layers while retaining each
consumer's own scoring and selection. Broad producer pools cover substantially more
consumer choices than final-selection overlap suggests across three model families.
On DeepSeek-V4-Flash, three consumer layers per producer give
$1.14$--$1.40\times$ serving speedups and $14.8$--$27.4\%$ lower GPU energy across four
suites, with observed accuracy changes of $-1.49$ to $+0.29$ points in single runs.
More extensive sharing exposes quality limits. Recurring selected KV positions across
prefill queries and layers further motivate Top-$G$, whose construction and deployment
require continued study. The results identify candidate working sets as a useful unit
for inference-time reuse.

\section*{Acknowledgements}
\label{sec:acknowledgements}
We thank the reviewers, advisors, and supporters for their valuable feedback and
support. We are grateful to NVIDIA Corporation for supporting this work. We
especially thank Ritchie Zhao and Bita Darvish Rouhani, as well as our managers
Yashwant Marathe, Bhargava Gopireddy, and Sachin Idgunji, for their insight,
guidance, and support.

\clearpage
\typeout{MLIPS_MAIN_PAGES=\number\numexpr\value{page}-1\relax}
\nobalance
\begingroup\small\raggedright
\bibliographystyle{plainnat}
\bibliography{mlipsrefs}

\begin{thebibliography}{17}
\providecommand{\natexlab}[1]{#1}
\providecommand{\url}[1]{\texttt{#1}}
\expandafter\ifx\csname urlstyle\endcsname\relax
  \providecommand{\doi}[1]{doi: #1}\else
  \providecommand{\doi}{doi: \begingroup \urlstyle{rm}\Url}\fi

\bibitem[Bai et~al.(2025)Bai, Tu, Zhang, Peng, Wang, Lv, Cao, Xu, Hou, Dong,
  et~al.]{longbenchv2}
Yushi Bai, Shangqing Tu, Jiajie Zhang, Hao Peng, Xiaozhi Wang, Xin Lv, Shulin
  Cao, Jiazheng Xu, Lei Hou, Yuxiao Dong, et~al.
\newblock Longbench v2: Towards deeper understanding and reasoning on realistic
  long-context multitasks.
\newblock In \emph{Proceedings of the 63rd Annual Meeting of the Association
  for Computational Linguistics (Volume 1: Long Papers)}, pages 3639--3664,
  2025.

\bibitem[Bai et~al.(2026)Bai, Dong, Jiang, Lv, Du, Zeng, Tang, and
  Li]{indexcache}
Yushi Bai, Qian Dong, Ting Jiang, Xin Lv, Zhengxiao Du, Aohan Zeng, Jie Tang,
  and Juanzi Li.
\newblock {IndexCache}: Accelerating sparse attention via cross-layer index
  reuse.
\newblock \emph{arXiv preprint arXiv:2603.12201}, 2026.
\newblock URL \url{https://arxiv.org/abs/2603.12201}.

\bibitem[Brandon et~al.(2024)Brandon, Mishra, Nrusimha, Panda, and
  Ragan-Kelley]{brandon2024cla}
William Brandon, Mayank Mishra, Aniruddha Nrusimha, Rameswar Panda, and
  Jonathan Ragan-Kelley.
\newblock Reducing transformer key-value cache size with cross-layer attention.
\newblock \emph{Advances in Neural Information Processing Systems},
  37:\penalty0 86927--86957, 2024.

\bibitem[{DeepSeek-AI}(2025)]{deepseekv32}
{DeepSeek-AI}.
\newblock {DeepSeek-V3.2}: Pushing the frontier of open large language models.
\newblock \emph{arXiv preprint arXiv:2512.02556}, 2025.
\newblock URL \url{https://arxiv.org/abs/2512.02556}.

\bibitem[{DeepSeek-AI}(2026{\natexlab{a}})]{deepseekv4}
{DeepSeek-AI}.
\newblock Deepseek-v4: Towards highly efficient million-token context
  intelligence.
\newblock Technical report, DeepSeek, 2026{\natexlab{a}}.
\newblock Model checkpoints:
  \url{https://huggingface.co/collections/deepseek-ai/deepseek-v4}.

\bibitem[{DeepSeek-AI}(2026{\natexlab{b}})]{deepseekv41}
{DeepSeek-AI}.
\newblock {DeepSeek-V4.1-Flash}: Pushing the limits of {KV} cache compression.
\newblock \emph{arXiv preprint arXiv:2609.19969}, 2026{\natexlab{b}}.
\newblock URL \url{https://arxiv.org/abs/2609.19969}.

\bibitem[Hsieh et~al.(2024)Hsieh, Sun, Kriman, Acharya, Rekesh, Jia, and
  Ginsburg]{ruler}
Cheng-Ping Hsieh, Simeng Sun, Samuel Kriman, Shantanu Acharya, Dima Rekesh, Fei
  Jia, and Boris Ginsburg.
\newblock Ruler: What's the real context size of your long-context language
  models?
\newblock In \emph{First Conference on Language Modeling}, 2024.

\bibitem[Lai et~al.(2026)Lai, Xu, Yang, Chen, Xu, Zeng, Li, Sun, Zhu, Zhang,
  Hu, Li, Gao, Li, Zhu, Zhou, and Zhao]{msa}
Xunhao Lai, Weiqi Xu, Yufeng Yang, Qiaorui Chen, Yang Xu, Lunbin Zeng, Xiaolong
  Li, Haohai Sun, Haichao Zhu, Vito Zhang, Jinkai Hu, Jiayao Li, Rui Gao, Zekun
  Li, Songquan Zhu, Jingkai Zhou, and Pengyu Zhao.
\newblock {MiniMax} sparse attention.
\newblock \emph{arXiv preprint arXiv:2606.13392}, 2026.
\newblock URL \url{https://arxiv.org/abs/2606.13392}.

\bibitem[Liu et~al.(2026)Liu, Cheng, Niu, Su, Xue, Liu, Yu, and Zhu]{pivot}
Hong Liu, Yuan Cheng, Lin Niu, Yi~Su, Yufei Xue, Anmin Liu, Guanghua Yu, and
  Jianchen Zhu.
\newblock {PIVOT}: Efficient query-group indexing for token-level sparse
  attention.
\newblock \emph{arXiv preprint arXiv:2607.24593}, 2026.
\newblock URL \url{https://arxiv.org/abs/2607.24593}.

\bibitem[{Meituan LongCat Team}(2026)]{lsa}
{Meituan LongCat Team}.
\newblock {LongCat} sparse attention: Taming the lightning via streaming-aware
  hierarchical cross-layer indexing.
\newblock \emph{arXiv preprint arXiv:2608.01662}, 2026.
\newblock URL \url{https://arxiv.org/abs/2608.01662}.

\bibitem[{OpenAI}(2025{\natexlab{a}})]{graphwalks}
{OpenAI}.
\newblock {GraphWalks}.
\newblock Dataset, 2025{\natexlab{a}}.
\newblock URL \url{https://huggingface.co/datasets/openai/graphwalks}.

\bibitem[{OpenAI}(2025{\natexlab{b}})]{mrcr}
{OpenAI}.
\newblock Multi-round coreference resolution (mrcr).
\newblock \url{https://huggingface.co/datasets/openai/mrcr},
  2025{\natexlab{b}}.

\bibitem[Sun et~al.(2024)Sun, Dong, Zhu, Huang, Wang, Ma, Zhang, Wang, and
  Wei]{sun2024yoco}
Yutao Sun, Li~Dong, Yi~Zhu, Shaohan Huang, Wenhui Wang, Shuming Ma, Quanlu
  Zhang, Jianyong Wang, and Furu Wei.
\newblock You only cache once: Decoder-decoder architectures for language
  models.
\newblock \emph{Advances in Neural Information Processing Systems},
  37:\penalty0 7339--7361, 2024.

\bibitem[Sun et~al.(2026)Sun, Zhang, Dong, Wang, and Wei]{clsa}
Yutao Sun, Yanqi Zhang, Li~Dong, Jianyong Wang, and Furu Wei.
\newblock You only index once: Cross-layer sparse attention with shared
  routing.
\newblock \emph{arXiv preprint arXiv:2606.06467}, 2026.
\newblock URL \url{https://arxiv.org/abs/2606.06467}.

\bibitem[Xu et~al.(2026)Xu, Meng, Jiang, Wang, Zhou, Wang, Wu, Pan, Tang, Pei,
  Liu, Yin, Sun, and Zhang]{hisa}
Yufei Xu, Fanxu Meng, Fan Jiang, Yuxuan Wang, Ruijie Zhou, Zhaohui Wang, Jiexi
  Wu, Zhixin Pan, Xiaojuan Tang, Wenjie Pei, Tongxuan Liu, Di~Yin, Xing Sun,
  and Muhan Zhang.
\newblock {HISA}: Efficient hierarchical indexing for fine-grained sparse
  attention.
\newblock \emph{arXiv preprint arXiv:2603.28458}, 2026.
\newblock URL \url{https://arxiv.org/abs/2603.28458}.

\bibitem[{Z.ai}(2026{\natexlab{a}})]{glm51}
{Z.ai}.
\newblock {GLM-5.1}: Model card.
\newblock \url{https://huggingface.co/zai-org/GLM-5.1}, 2026{\natexlab{a}}.
\newblock Released 2026-04-07; architecture as described in the GLM-5 technical
  report (arXiv:2602.15763).

\bibitem[{Z.ai}(2026{\natexlab{b}})]{glm52}
{Z.ai}.
\newblock {GLM-5.2}: Built for long-horizon tasks.
\newblock Official technical release post, June 2026{\natexlab{b}}.
\newblock URL \url{https://huggingface.co/blog/zai-org/glm-52-blog}.

\end{thebibliography}
\endgroup
\clearpage
\appendix
\typeout{MLIPS_APPENDIX_START_PAGE=\thepage}

% Embedded lossless copies of the original appendix plots.

% Source fragment: sections/appendix_methodology.tex
\section{Scoring Model and Detailed Reductions}
\label{app:cost}

\subsection{Indexer arithmetic and its scope}
For query position $t$ and visible key $j$, the Lightning Indexer computes
\begin{equation}
s[t,j]=\sum_{u=1}^{H_i}w[t,u]\,
\operatorname{ReLU}\!\left(\langle\mathbf q[t,u,:],\mathbf k[j,:]\rangle\right).
\label{eq:indexer-score}
\end{equation}
There is one key vector per indexed position shared across query heads.
Scoring $Q$ queries against $R$ candidates has matrix dimensions
\begin{equation}
[B,QH_i,D_i]\,[B,D_i,R]\ \longrightarrow\ [B,QH_i,R],
\label{eq:indexer-bmm}
\end{equation}
where $B$ is the sequence batch size. Counting multiply--accumulate as two operations
gives $I(Q,R)=2H_iD_iQR$ per sequence. DeepSeek-V4-Flash uses
$H_i=64$, $D_i=128$, $n=21$ CSA layers, $r=4$, $K=512$, and a local window
of $W=128$ tokens. Thus one query--candidate dot-product calculation costs
$16{,}384$ operations in this model.

ReLU, weighting, and head reduction add $O(H_iQR)$ arithmetic. Selection, gathering,
cache access, and transfers are also excluded from the ratios. Full prefill has
quadratic scoring growth in $L$; decode has linear growth in cached length per query.
The rectangular model approximates causal pair counts and boundary effects.
All reduction tables use $L=2^{20}$ tokens, $N=L/r=262{,}144$,
$M=16{,}384$, and hence $N/M=16$. This makes the nominal 1M point
explicit; the ratios are not independent of context length.

\subsection{Coverage under nested pools}
\begin{proposition}[Nested-pool coverage]
\label{prop:dominance}
For fixed producer and consumer score vectors, if
$\mathcal P_K\subseteq\mathcal P_M$, then
$|\mathcal C_K\cap\mathcal P_K|\leq|\mathcal C_K\cap\mathcal P_M|$.
\end{proposition}
\begin{proof}
Intersecting nested sets with the same set preserves inclusion.
\end{proof}
This bounds index coverage only. It does not compare attention distributions or
task accuracy, and it does not apply to independently constructed pools unless the
required nesting holds.

\subsection{Layer and query schedules}
For $p=\lceil n/(h+1)\rceil$ and $c=n-p$, a decode cycle contains $qh+1$ steps.
Each producer scores $N$ once and $M$ on the remaining $qh$ steps; each consumer
scores $M$ on every step. The cycle work is proportional to
$pN+[p\,qh+c(qh+1)]M$, yielding Equation~\ref{eq:2d-exact}.

For uniform groups with $h+1$ dividing $n$, a pool serves
$T=(h+1)(qh+1)$ slots. The reduction becomes
\begin{equation}
R(T)=\frac{TN}{N+(T-1)M},
\qquad \lim_{T\to\infty}R(T)=N/M.
\label{eq:2d-reduction}
\end{equation}
A $2\times2$ uniform tile gives $3.37\times$, while the exact 21-layer
$(h,qh)=(1,1)$ schedule gives $3.25\times$. For 1D sharing at $h=1,2,3$,
the exact producer counts are $11,7,6$ and the consumer counts are $10,14,15$.
Table~\ref{tab:ops-2d} includes 1D in its $qh=0$ column.

\begin{table}[t]
\centering\small
\caption{Modeled indexer scoring-work reduction for the exact 21-layer stack.
The $qh=0$ column gives layerwise reuse alone. Values are reduction factors.}
\label{tab:ops-1d}\label{tab:ops-2d}
\begin{tabular}{lrrrr}
\toprule
 & $qh=0$ & $qh=1$ & $qh=2$ & $qh=3$\\
\midrule
$h=0$ & $1.00$ & $1.88$ & $2.67$ & $3.37$\\
$h=1$ & $1.81$ & $3.25$ & $4.42$ & $5.40$\\
$h=2$ & $2.67$ & $4.57$ & $6.00$ & $7.11$\\
$h=3$ & $3.03$ & $5.09$ & $6.59$ & $7.72$\\
\bottomrule
\end{tabular}
\end{table}

\subsection{Coarse-producer accounting}
The producer scores $N/P$ representatives and performs its own fine pass over
$bP=M$ candidates. Its build-step cost is $I(Q,N/P)+I(Q,M)$.
Every layer performs one fine pass on every step, so the cycle cost is proportional
to $pN/P+n(qh+1)M$, yielding Equation~\ref{eq:cp-exact}.
The grids in Table~\ref{tab:ops-cp} include configurations with substantial
observed quality losses; a large modeled reduction is not an operating-point
recommendation.

% Source fragment: figures/fig_hier.tex
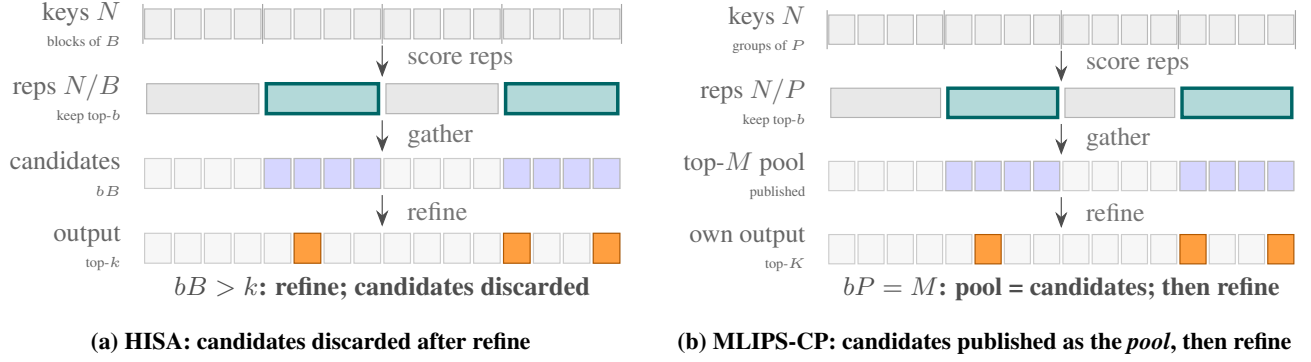
\begin{figure*}[t]
\centering
\tikzset{
  k/.style     = {draw=gray!55, fill=gray!12, minimum width=3.2mm, minimum height=3.6mm,
                  inner sep=0pt},
  kdim/.style  = {draw=gray!40, fill=gray!6,  minimum width=3.2mm, minimum height=3.6mm,
                  inner sep=0pt},
  kcand/.style = {draw=gray!55, fill=blue!16, minimum width=3.2mm, minimum height=3.6mm,
                  inner sep=0pt},
  ksel/.style  = {draw=orange!70!black, fill=orange!75, minimum width=3.2mm,
                  minimum height=3.6mm, inner sep=0pt},
  rep/.style   = {draw=gray!60, fill=gray!18, minimum width=13.5mm, minimum height=3.6mm,
                  inner sep=0pt},
  repsel/.style= {draw=teal!80!black, fill=teal!30, very thick, minimum width=13.5mm,
                  minimum height=3.6mm, inner sep=0pt},
  lb/.style    = {font=\small, text=black!60, align=right},
  an/.style    = {font=\small, text=black!55},
  st/.style    = {font=\small\bfseries, text=black!70},
  ar/.style    = {-{Stealth[length=2mm]}, gray!60!black},
  sep/.style   = {gray!70, thin},
}

\begin{minipage}[t]{0.48\textwidth}
\centering
\resizebox{\linewidth}{!}{%
\begin{tikzpicture}[x=3.6mm, y=1cm]
  \foreach \i in {0,...,15}{ \node[k] at (\i,3.00) {}; }
  \foreach \b in {0,1,2,3}{ \draw[sep] ({4*\b-0.5},3.24) -- ({4*\b-0.5},2.76); }
  \draw[sep] (15.5,3.24) -- (15.5,2.76);
  \node[lb, anchor=east] at (-0.9,3.00) {keys $N$\\[-2pt] \tiny blocks of $B$};

  \draw[ar] (7.5,2.74) -- (7.5,2.36);
  \node[an, anchor=west] at (8.0,2.55) {score reps};
  \node[rep]    at (1.5,2.10) {};
  \node[repsel] at (5.5,2.10) {};
  \node[rep]    at (9.5,2.10) {};
  \node[repsel] at (13.5,2.10) {};
  \node[lb, anchor=east] at (-0.9,2.10) {reps $N/B$\\[-2pt] \tiny keep top-$b$};

  \draw[ar] (7.5,1.84) -- (7.5,1.46);
  \node[an, anchor=west] at (8.0,1.65) {gather};
  \foreach \i in {0,...,15}{ \node[kdim] at (\i,1.20) {}; }
  \foreach \i in {4,5,6,7,12,13,14,15}{ \node[kcand] at (\i,1.20) {}; }
  \node[lb, anchor=east] at (-0.9,1.20) {candidates\\[-2pt] \tiny $bB$};

  \draw[ar] (7.5,0.94) -- (7.5,0.56);
  \node[an, anchor=west] at (8.0,0.75) {refine};
  \foreach \i in {0,...,15}{ \node[kdim] at (\i,0.30) {}; }
  \foreach \i in {5,12,15}{ \node[ksel] at (\i,0.30) {}; }
  \node[lb, anchor=east] at (-0.9,0.30) {output\\[-2pt] \tiny top-$k$};

  \node[st] at (7.5,-0.18) {$bB > k$: refine; candidates discarded};
\end{tikzpicture}}
\\[3pt]
{\small\textbf{(a) HISA: candidates discarded after refine}}
\end{minipage}
\hfill
\begin{minipage}[t]{0.48\textwidth}
\centering
\resizebox{\linewidth}{!}{%
\begin{tikzpicture}[x=3.6mm, y=1cm]
  \foreach \i in {0,...,15}{ \node[k] at (\i,3.00) {}; }
  \foreach \b in {0,1,2,3}{ \draw[sep] ({4*\b-0.5},3.24) -- ({4*\b-0.5},2.76); }
  \draw[sep] (15.5,3.24) -- (15.5,2.76);
  \node[lb, anchor=east] at (-0.9,3.00) {keys $N$\\[-2pt] \tiny groups of $P$};

  \draw[ar] (7.5,2.74) -- (7.5,2.36);
  \node[an, anchor=west] at (8.0,2.55) {score reps};
  \node[rep]    at (1.5,2.10) {};
  \node[repsel] at (5.5,2.10) {};
  \node[rep]    at (9.5,2.10) {};
  \node[repsel] at (13.5,2.10) {};
  \node[lb, anchor=east] at (-0.9,2.10) {reps $N/P$\\[-2pt] \tiny keep top-$b$};

  \draw[ar] (7.5,1.84) -- (7.5,1.46);
  \node[an, anchor=west] at (8.0,1.65) {gather};
  \foreach \i in {0,...,15}{ \node[kdim] at (\i,1.20) {}; }
  \foreach \i in {4,5,6,7,12,13,14,15}{ \node[kcand] at (\i,1.20) {}; }
  \node[lb, anchor=east] at (-0.9,1.20) {top-$M$ pool\\[-2pt] \tiny published};

  \draw[ar] (7.5,0.94) -- (7.5,0.56);
  \node[an, anchor=west] at (8.0,0.75) {refine};
  \foreach \i in {0,...,15}{ \node[kdim] at (\i,0.30) {}; }
  \foreach \i in {5,12,15}{ \node[ksel] at (\i,0.30) {}; }
  \node[lb, anchor=east] at (-0.9,0.30) {own output\\[-2pt] \tiny top-$K$};

  \node[st] at (7.5,-0.18) {$bP = M$: pool = candidates; then refine};
\end{tikzpicture}}
\\[3pt]
{\small\textbf{(b) MLIPS-CP: candidates published as the \emph{pool}, then refine}}
\end{minipage}

\caption{Coarse-to-fine search and shared candidate production. Both score
representatives and gather the highest-scoring groups. HISA refines the gathered
candidates for one layer; MLIPS-CP publishes them as a pool and retains a fine
selection for each producer and consumer. Group sizes differ in practice; the
diagram uses equal granularity for illustration.}
\label{fig:hier}
\end{figure*}

\begin{table*}[t]
\centering\small
\caption{Modeled scoring-work reductions with a coarse producer, $b=M/P$.
Values are factors relative to full indexing; representative construction
and memory costs are excluded.}
\label{tab:ops-cp}
\begin{tabular}{lrrrr@{\hspace{1cm}}lrrrr}
\toprule
\multicolumn{5}{c}{$P=2$} & \multicolumn{5}{c}{$P=4$}\\
 & $qh=0$ & $1$ & $2$ & $3$ & & $qh=0$ & $1$ & $2$ & $3$\\
\midrule
$h=0$&1.78&3.20&4.36&5.33 & $h=0$&3.20&5.33&6.86&8.00\\
$h=1$&3.08&5.17&6.68&7.81 & $h=1$&5.17&7.81&9.42&10.50\\
$h=2$&4.36&6.86&8.47&9.60 & $h=2$&6.86&9.60&11.08&12.00\\
$h=3$&4.87&7.47&9.08&10.18 & $h=3$&7.47&10.18&11.59&12.44\\
\midrule
\multicolumn{5}{c}{$P=8$} & \multicolumn{5}{c}{$P=16$}\\
 & $qh=0$ & $1$ & $2$ & $3$ & & $qh=0$ & $1$ & $2$ & $3$\\
\midrule
$h=0$&5.33&8.00&9.60&10.67 & $h=0$&8.00&10.67&12.00&12.80\\
$h=1$&7.81&10.50&11.86&12.68 & $h=1$&10.50&12.68&13.62&14.15\\
$h=2$&9.60&12.00&13.09&13.71 & $h=2$&12.00&13.71&14.40&14.77\\
$h=3$&10.18&12.44&13.44&14.00 & $h=3$&12.44&14.00&14.61&14.93\\
\bottomrule
\end{tabular}
\end{table*}

\paragraph{Representative-score dilution.}
For a single head and a mean representative $\bar{\mathbf k}=P^{-1}\sum_j\mathbf k_j$,
convexity gives
\begin{equation}
\begin{split}
\operatorname{ReLU}(\langle\mathbf q,\bar{\mathbf k}\rangle)
&\leq P^{-1}\sum_j\operatorname{ReLU}(\langle\mathbf q,\mathbf k_j\rangle)\\
&\leq \max_j\operatorname{ReLU}(\langle\mathbf q,\mathbf k_j\rangle).
\end{split}
\label{eq:dilution}
\end{equation}
Cancellation can suppress a representative even when its group contains a high-scoring
key. This single-head inequality describes a possible ranking failure; it does not
prove that full multi-head rankings or task accuracy degrade monotonically in $P$.
Such a claim requires score-level evidence, the relevant head-weight assumptions,
and matched serving controls.

\newpage
\section{Implementation Details}
\label{app:implementation}

We implement producer and consumer kernels for prefill and decode in vLLM's
DeepSeek-V4 sparse-attention path. The checkpoint weights, attention operator,
scheduler, and KV-cache path are retained. Baseline and MLIPS use the same software,
with configuration flags selecting the indexer path.

\paragraph{Integration and buffers.}
The implementation overrides the Lightning-Indexer selection stage and a gather module
on vLLM's DeepSeek-V4 path. Producers and consumers have separate prefill and decode
kernel sites. An exact producer obtains its own top-$K$ and the wider top-$M$ pool
from one scoring pass; a consumer gathers and scores its own keys at the pooled
positions. A pool row contains $M$ \texttt{int32} indices: $64$~KiB at $M=16{,}384$.
Total metadata depends on query-row and persisted-pool counts, even though
each row's size is fixed at a given $M$.

During prefill, producers overwrite a transient candidate buffer before their
consumers use it. During decode, per-producer buffers retain pools across rebuild
cycles. Rows with at most $M$ compressed keys use stock indexing, corresponding to
$4M=65{,}536$ tokens at $r=4$. This threshold helps explain why workloads with more
short contexts obtain smaller gains.

\paragraph{Coarse representatives.}
Representatives are computed in fp32 from dequantized keys and quantized into the
indexer key format. In prefill, running sums track keys appended to each group.
In decode, completed representatives are addressed through the KV-page block table,
with a running sum for an incomplete group. One $D_i=128$ MXFP4 representative uses
$64$ value bytes and four scale bytes. There is one representative per $P$ keys;
actual allocated overhead additionally depends on page layout and temporary
state. The cache can be replaced by on-the-fly pooling at the cost of rereading the
original keys. Coarse production is used above $32{,}768$ compressed entries
(131,072 tokens); shorter rows use exact production.

\paragraph{Decode graphs and recency.}
The persistent path captures rebuild and reuse variants separately. A field in
the batch descriptor selects the replay outside the CUDA graph, avoiding a Python
decision fixed at capture time. Evaluation disables speculative decoding.
Within-step speculative grouping and reuse across successive steps are different
schedules; their combined acceptance behavior and performance are not measured here.

For longer reuse windows, one possible policy is to admit newly materialized
compressed keys and remove low-priority pool entries. Such a policy needs its own
selection rule and validation. It does not address older keys that become relevant
after the pool was built.

\paragraph{Selection checks.}
The exact producer's top-$K$ differs from stock by at most $0.02\%$ on the
million-token prompts in the reported check, with mismatches attributed to equal
scores. A probe at $(h,qh)=(1,1)$ performs a full scan on reuse steps and compares
it with selection from the stored pool. Over ten RULER-1M prompts, mean overlap
is $0.974$ and median overlap is $1.000$. This quantifies the probe's selection
fidelity; it is not an end-task accuracy test or a general guarantee against
implementation errors.

% Source fragment: sections/appendix_recall.tex
\section{Full Recall Measurements}
\label{app:recall}

\paragraph{Captures and native units.}
The recall study uses the ten largest prompts from each of LongBench-v2,
RULER-1M, and MRCR-v2 for each captured model family. It records four consecutive
query positions every 256 positions. DeepSeek-V4-Flash indexes compressed entries
and selects $K=512$; MiniMax-M3 selects 16 blocks of 128 tokens; GLM-5.1 selects
2,048 tokens~\citep{deepseekv4,msa,glm51}. The configured broad pools contain
16,384 compressed entries, 128 blocks, and 16,384 tokens, respectively.
These widths describe different token coverage and are not equal-cost comparisons.
GLM-5.1 captures end near 198K tokens; the other two families include up-to-1M inputs.

\paragraph{Aggregation and distances.}
The layer-sharing surfaces average over consumers in a schedule with $h$
consumers per producer. Producers are excluded because their own coverage is trivial.
For $h>1$, a cell averages multiple exact layer distances $d$ and a different producer
placement from neighboring schedules. The query-axis surfaces instead use explicit
producer--consumer layer and query offsets. Their same-query column therefore
need not exactly reproduce a schedule-average cell.

Position bands are ranges within prompts. A long prompt contributes queries to
several bands; the below-200K figure is a subset of the all-position figure.
The broad-pool columns use nested prefixes of a producer ranking when available.
MiniMax-M3 forces a local block into its selection. Its effective-containment
convention counts the guaranteed local hit; the difference from unadjusted
containment is about 0.010--0.016 at 128 blocks and at most 0.035 at narrow widths.

\paragraph{Scope of the evidence.}
These are count-based index recalls on sampled prefill activations, not attention-mass
retention. Four-query bursts support query offsets only up to three.
Within a fixed, identical causal prefix, prefill and decode evaluate the same
indexing computation, but generated-token distributions and multi-step approximation
effects require separate tests. The decode probe in Appendix~\ref{app:implementation}
provides a limited additional check.

\begin{figure*}[p]
\centering
\includegraphics[width=\textwidth,height=0.83\textheight,keepaspectratio]{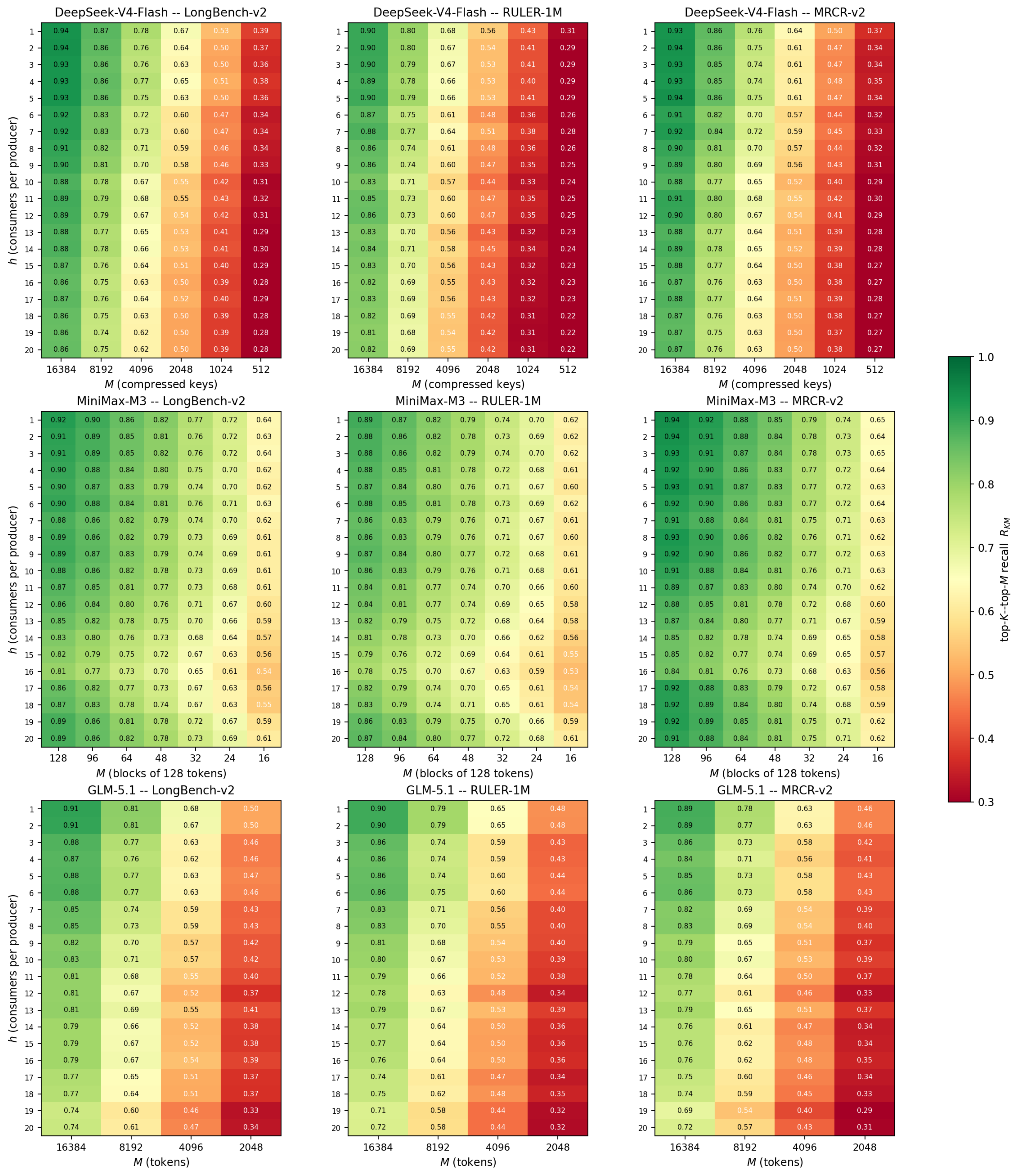}
\caption{Layer-sharing recall below 200K query positions. Rows are model families,
columns are suites, and each panel varies schedule $h$ and producer-pool width $M$.
The rightmost column is the native top-$K$ selection overlap; wider columns show
pool coverage. Units differ across families as labeled.}
\label{fig:recallA}
\end{figure*}

\begin{figure*}[p]
\centering
\includegraphics[width=\textwidth,height=0.83\textheight,keepaspectratio]{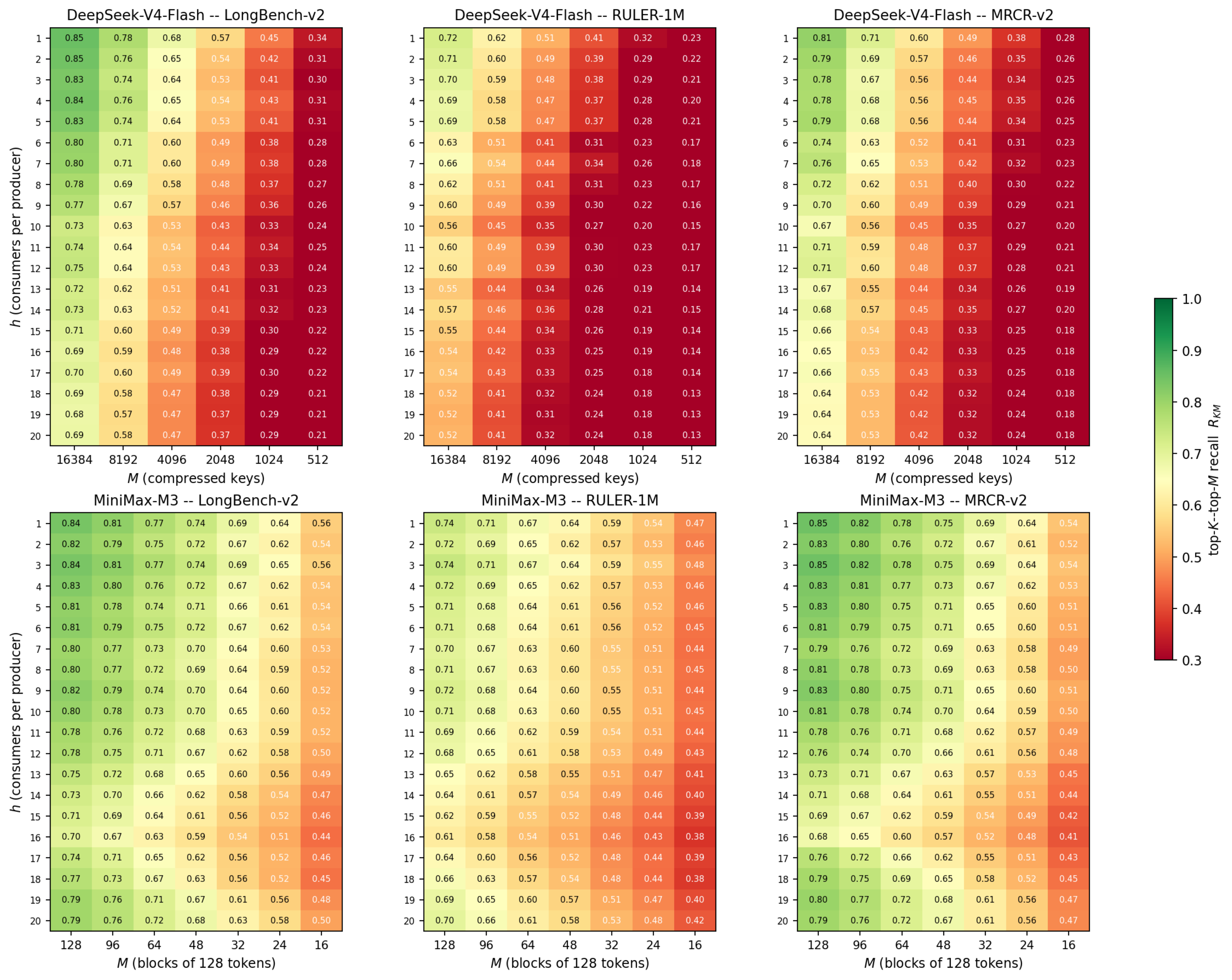}
\caption{Layer-sharing recall over all sampled positions in up-to-1M captures.
Broad candidate pools retain more consumer choices than final-selection inheritance
on these captures; neither measurement alone establishes downstream accuracy.
Axes follow Figure~\ref{fig:recallA}.}
\label{fig:recallB}
\end{figure*}

\begin{figure*}[p]
\centering
\includegraphics[width=\textwidth,height=0.83\textheight,keepaspectratio]{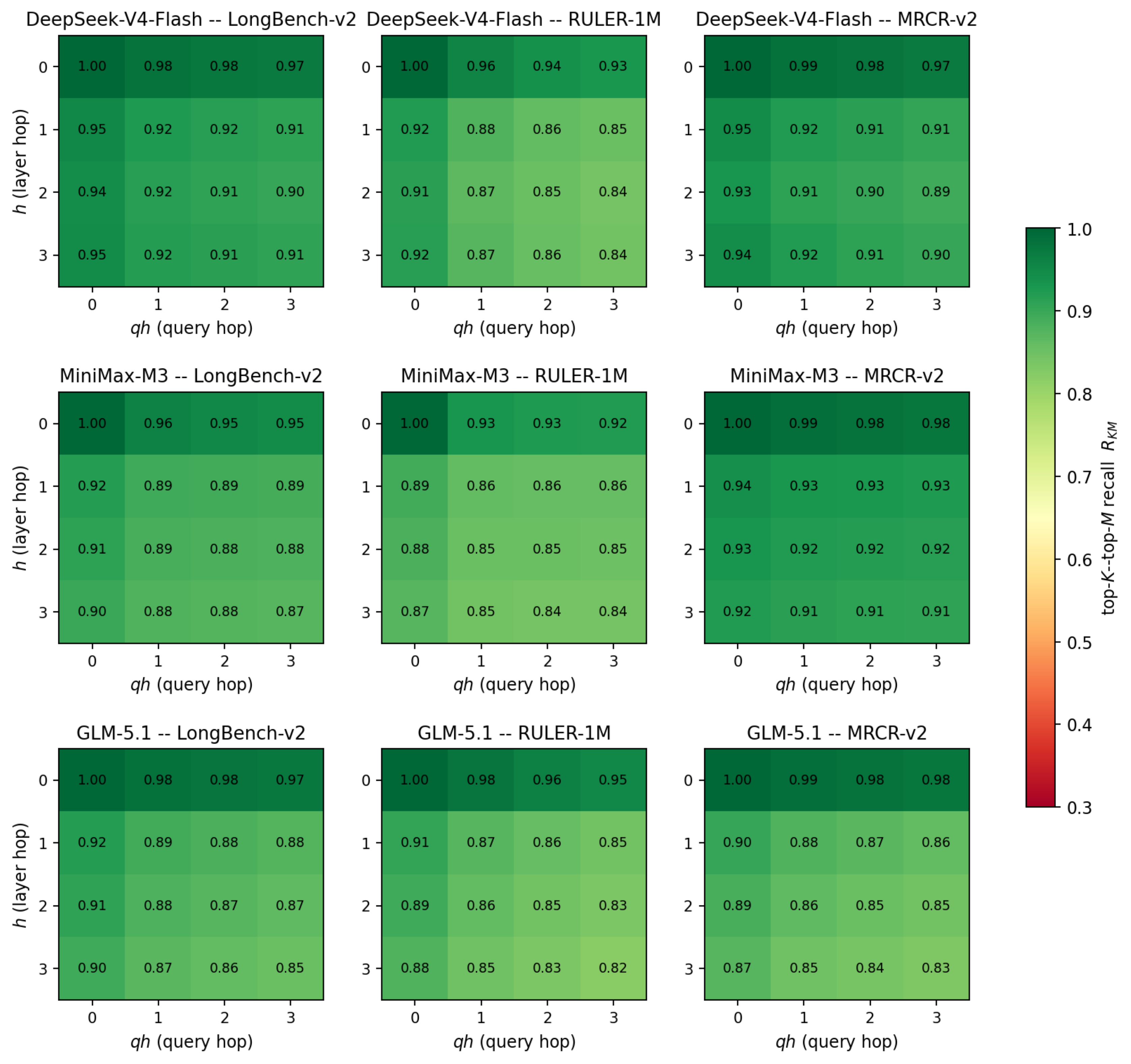}
\caption{Joint layer/query candidate coverage below 200K query positions.
Each panel varies exact sparse-layer distance (labeled $h$ in the original plot,
$d$ in the manuscript) and query offset $qh$ at the configured broad pool.
The zero-layer-distance row isolates cross-query reuse.}
\label{fig:recall2dA}
\end{figure*}

\begin{figure*}[p]
\centering
\includegraphics[width=\textwidth,height=0.83\textheight,keepaspectratio]{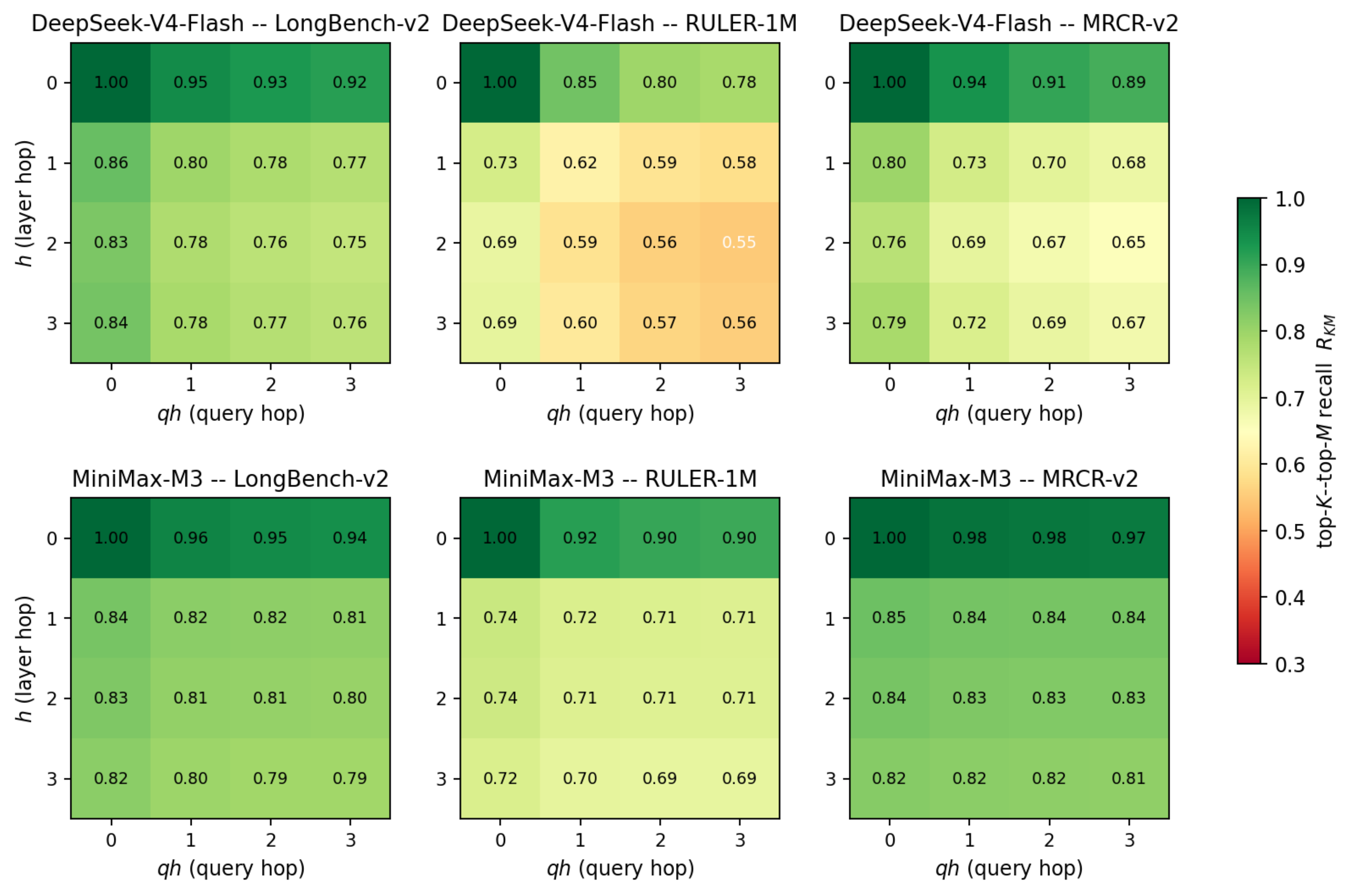}
\caption{Joint layer/query candidate coverage over all sampled positions in
up-to-1M captures. Exact layer distance is labeled $h$ in this original plot.
The self pair is 1.00; later queries retain substantial but imperfect coverage.
The captured query offsets end at three.}
\label{fig:recall2dB}
\end{figure*}

\paragraph{Schedule variation.}
Larger sharing groups generally lower coverage, but averages can rise locally when
producer placements change. For example, MiniMax-M3 at its 128-block pool gives
schedule-average coverage of 0.81 at $h=16$ and 0.89 at $h=19$.
These observations motivate evaluating the actual schedule, rather than treating
$h$ as a smooth distance parameter. Original plots display schedules through
$h=20$; captures extend to $h=56$ on MiniMax-M3 and $h=77$ on GLM-5.1.

\clearpage
% Source fragment: sections/appendix_evaluation.tex
\section{Serving Setup and Complete Results}
\label{app:evaluation}

\subsection{Model, software, and measurement protocol}
\label{app:setup}

The served checkpoint is the released DeepSeek-V4-Flash. Its routed MoE
expert weights are FP4, with E8M0 scales over blocks of $32$ values,
executed through vLLM's MXFP4 kernels. Attention projections, the shared
expert, and the indexer query projection use FP8 E4M3 weights with
$128\times128$ blocks. The MoE router, indexer compressor, and normalization
weights are bf16. Computation runs in bf16, with activations quantized on
the fly to FP8 for the FP8 GEMMs and to MXFP8 for the expert GEMMs.
The KV cache is FP8 and the indexer key cache is FP4.

We use vLLM 0.22.1 in the NVIDIA vLLM 26.06 container. Each run occupies
four GPUs of a single NVIDIA GB300 node, with a $1{,}400$~W power limit
per GPU and tensor parallelism $4$. Expert parallelism and speculative
decoding are off; CUDA graphs are on. The chunked-prefill budget is
$8{,}192$ tokens, the maximum number of concurrent sequences is $24$, the
context limit is $1{,}048{,}576$ tokens, and prefix caching is off.
The client maintains $24$ requests in flight. The exact-producer pool
has $M=16{,}384$ entries. Rows below the pooling threshold use the stock
path as described in Appendix~\ref{app:implementation}.

The four suites contain $503$ LongBench-v2 items, $130$ RULER-1M items
(ten in each of $13$ task categories), $784$ MRCR-v2 8-needle items, and
$1{,}150$ GraphWalks items. Scores are reported on a $0$--$100$ scale,
with $\Delta$acc in percentage points. Speedup is stock workload wall
time divided by configuration wall time. Energy saved is one minus the
configuration-to-stock GPU-energy ratio. The $\Delta$power column is the
relative change in mean total GPU power. Retaining these separate
quantities matters because shorter runs need not have lower mean power.

Each GPU's power is sampled every $500$~ms with \texttt{nvidia-smi};
energy integrates the four-GPU total from the end of warm-up to workload
completion. Boot, JIT compilation, and CUDA-graph capture are excluded.
The measurement does not include host CPUs, host memory, or network
equipment. A second stock boot gives a limited reproducibility reference:
suite scores differed by approximately one to two points or less, wall
time by at most $1.3\%$, and GPU energy by at most $0.3\%$.
Every configuration row remains a single serving run, so these two stock
boots do not provide uncertainty estimates for treatment effects.

The stock configurations disable all MLIPS paths by flag. The exact
sharing grid and max-representative experiment use the first stock rows
shown in Tables~\ref{tab:eval-1d} and~\ref{tab:eval-rep}. The
mean-representative and group-size campaign uses the separate stock rows
shown in Tables~\ref{tab:eval-cp} and~\ref{tab:eval-psweep}. We preserve
these baselines rather than pooling measurements across campaigns.
The coarse-producer comparisons do not include matched repeats or complete
per-arm launch configurations in the current reporting. We interpret the
group-size trend as an observed quality tradeoff, with representative
dilution as one possible explanation. Verifying the cause requires matched
launch configurations and repeated controls.

\subsection{Complete results and diagnostic observations}
\label{app:results}

\paragraph{Layer sharing.}
Table~\ref{tab:eval-1d} includes both the $h=1,2,3$ configurations in the
main paper and deeper schedules $h=5,6$. Every measured setting improves
wall time and GPU energy relative to its stock row. At $h=6$, speedups
are $1.310$, $1.572$, $1.233$, and $1.432\times$ on LongBench-v2,
RULER-1M, MRCR-v2, and GraphWalks, respectively. RULER-1M loses $3.33$
and $4.42$ points at $h=5,6$; the other suites have smaller aggregate
score changes in these runs.

Per-task RULER summaries suggest particular sensitivity in
\texttt{niah\_multikey\_2} and \texttt{niah\_multikey\_3}, which contain
many distractor keys. The three single-needle tasks score $100$ for
stock and the reported layer-sharing settings, except for one prompt
at $h=6$. The multi-key task with one distractor type, multi-query
retrieval, and multi-value retrieval remain at $100$ or within one
reported prompt of that value. At $h\leq3$, the two more difficult
multi-key tasks change by at most one reported prompt; at $h=5,6$ they
lose one to two reported prompts each. At $h=6$, the source also reports
one-prompt losses in a single-needle task, frequent-word retrieval, and
each QA task. Other tasks remain within two reported prompts of stock;
stock variable tracking and common-word scores are $22$ and $0$.
These task patterns do not establish why a candidate restriction changes
an answer. In particular, per-task scores should not be converted into
fixed $0.77$-point suite increments without accounting for the tasks'
scoring rules.

GraphWalks BFS scores are $34.6$ for stock and $36.4$, $34.4$, $37.5$,
$33.5$, and $37.2$ for $h=1,2,3,5,6$. On items longer than $256$K tokens,
the reported paired score change at $h=2$ is $-5.6$ points; the
corresponding changes at $h=1,3$ are $-0.2$ and $+1.6$ points.
The aggregate and length-specific results therefore need not move
together. The $h=2$ configuration has one boot, and its cause has not
been established.

At $h=1$, energy savings are $8.9$, $16.9$, $6.6$, and $13.9\%$, whereas
the corresponding fractions of wall time saved are $6.2$, $15.0$,
$1.5$, and $11.6\%$. Mean GPU power falls by $2.2$--$5.2\%$.
That reduction is not universal: at $h=5,6$ on RULER-1M mean power rises
by approximately $3$--$4\%$, while a shorter run still reduces energy.

\begin{table*}[t]
\centering\footnotesize\setlength{\tabcolsep}{4pt}
\caption{Complete layer-sharing results ($qh=0$). Scores and changes use the stock row within each suite; energy is GPU-only and measured after warm-up. Every configuration is one serving run.}
\label{tab:eval-1d}
\begin{tabular}{llrrrrrrr}
\toprule
suite & config & acc & $\Delta$acc & wall (s) & speedup & energy (Wh) & saved & $\Delta$power \\
\midrule
LongBench-v2 & stock & $50.50$ & --- & $2{,}970$ & --- & $3{,}326$ & --- & --- \\
 & $h{=}1$ & $49.11$ & $-1.39$ & $2{,}787$ & $1.066\times$ & $3{,}029$ & $8.9\%$ & $-3.0\%$ \\
 & $h{=}2$ & $49.70$ & $-0.80$ & $2{,}556$ & $1.162\times$ & $2{,}818$ & $15.3\%$ & $-1.5\%$ \\
 & $h{=}3$ & $50.10$ & $-0.40$ & $2{,}481$ & $1.197\times$ & $2{,}756$ & $17.1\%$ & $-0.8\%$ \\
 & $h{=}5$ & $50.30$ & $-0.20$ & $2{,}352$ & $1.263\times$ & $2{,}645$ & $20.5\%$ & $+0.4\%$ \\
 & $h{=}6$ & $49.11$ & $-1.39$ & $2{,}267$ & $1.310\times$ & $2{,}588$ & $22.2\%$ & $+1.9\%$ \\
\midrule
RULER-1M & stock & $80.41$ & --- & $5{,}982$ & --- & $6{,}624$ & --- & --- \\
 & $h{=}1$ & $79.91$ & $-0.50$ & $5{,}084$ & $1.177\times$ & $5{,}507$ & $16.9\%$ & $-2.2\%$ \\
 & $h{=}2$ & $78.46$ & $-1.95$ & $4{,}456$ & $1.342\times$ & $4{,}960$ & $25.1\%$ & $+0.5\%$ \\
 & $h{=}3$ & $80.12$ & $-0.29$ & $4{,}278$ & $1.398\times$ & $4{,}807$ & $27.4\%$ & $+1.5\%$ \\
 & $h{=}5$ & $77.08$ & $-3.33$ & $3{,}976$ & $1.504\times$ & $4{,}534$ & $31.6\%$ & $+3.0\%$ \\
 & $h{=}6$ & $75.99$ & $-4.42$ & $3{,}805$ & $1.572\times$ & $4{,}392$ & $33.7\%$ & $+4.2\%$ \\
\midrule
MRCR-v2 & stock & $75.02$ & --- & $4{,}911$ & --- & $5{,}093$ & --- & --- \\
 & $h{=}1$ & $73.59$ & $-1.43$ & $4{,}839$ & $1.015\times$ & $4{,}755$ & $6.6\%$ & $-5.2\%$ \\
 & $h{=}2$ & $74.64$ & $-0.38$ & $4{,}396$ & $1.117\times$ & $4{,}423$ & $13.2\%$ & $-3.0\%$ \\
 & $h{=}3$ & $73.53$ & $-1.49$ & $4{,}312$ & $1.139\times$ & $4{,}341$ & $14.8\%$ & $-2.9\%$ \\
 & $h{=}5$ & $74.05$ & $-0.97$ & $4{,}151$ & $1.183\times$ & $4{,}197$ & $17.6\%$ & $-2.5\%$ \\
 & $h{=}6$ & $74.90$ & $-0.12$ & $3{,}984$ & $1.233\times$ & $4{,}097$ & $19.6\%$ & $-0.9\%$ \\
\midrule
GraphWalks & stock & $53.35$ & --- & $12{,}264$ & --- & $12{,}976$ & --- & --- \\
 & $h{=}1$ & $54.50$ & $+1.15$ & $10{,}847$ & $1.131\times$ & $11{,}178$ & $13.9\%$ & $-2.6\%$ \\
 & $h{=}2$ & $52.56$ & $-0.79$ & $9{,}504$ & $1.290\times$ & $10{,}128$ & $21.9\%$ & $+0.7\%$ \\
 & $h{=}3$ & $53.64$ & $+0.29$ & $9{,}248$ & $1.326\times$ & $9{,}898$ & $23.7\%$ & $+1.2\%$ \\
 & $h{=}5$ & $51.68$ & $-1.67$ & $8{,}812$ & $1.392\times$ & $9{,}430$ & $27.3\%$ & $+1.1\%$ \\
 & $h{=}6$ & $53.07$ & $-0.28$ & $8{,}562$ & $1.432\times$ & $9{,}191$ & $29.2\%$ & $+1.4\%$ \\
\bottomrule
\end{tabular}
\end{table*}

\paragraph{Persistent decode reuse.}
Tables~\ref{tab:eval-2d}, \ref{tab:eval-2d-ruler},
\ref{tab:eval-2d-mrcr}, and~\ref{tab:eval-2d-graphwalks}
give the entire $qh=1,2,3$ grid at $h=0,1,2,3$.
Together with Table~\ref{tab:eval-1d}, they retain the full measured
layer-and-query matrix. At a fixed $h\geq1$, the reported speedups across
$qh=0,1,2,3$ differ by at most $4.1\%$, and energy savings by at most
$0.6$ percentage points. The query-only configurations are no faster
than stock, reaching $0.960$--$1.000\times$.
This is consistent with limited opportunity for decode-only reuse in
these workloads; it does not establish identical execution cost.

The largest reported score losses include RULER-1M at $(h,qh)=(3,2)$
($-9.68$ points) and $(3,3)$ ($-6.18$), MRCR-v2 at $(1,3)$ ($-6.73$)
and $(3,2)$ ($-5.29$), and GraphWalks at $(3,2)$ ($-3.01$) and
$(3,3)$ ($-4.12$). RULER-1M is not monotone in query reuse at $h=3$:
the changes are $-5.44$, $-9.68$, and $-6.18$ points.
The corresponding task scores for \texttt{niah\_multikey\_2} and
\texttt{niah\_multikey\_3} are $40$ and $50$ at $(3,2)$, and $90$ and
$40$ at $(2,2)$. The source reports that single-needle, multi-query, and
multi-value tasks remain at $100$ or within one prompt in this grid.
The matrix contains single runs and does not identify the mechanism
behind individual task failures.

The isolated timing experiment in Table~\ref{tab:eval-decode} uses
$30$ synthetic prompts of approximately $0.9$M tokens, $24$ sessions
offered, and no prefix caching. Decode time is estimated as the
difference between a $40{,}000$-token generation and a one-token
generation. Each arm has one boot and two measured draws after
discarding a warm-up draw. It measures $h=0$ only; the reported
$3.3\%$ reduction at $qh=3$ is not a measured combined layer-and-query
speedup. This workload has no answer-quality evaluation.

\begin{table*}[t]
\centering\footnotesize\setlength{\tabcolsep}{4pt}
\caption{Persistent decode reuse on LongBench-v2 ($qh\geq1$). The $h=0$ rows use query reuse alone; $qh=0$ results appear in Table~\ref{tab:eval-1d}. Ratios use the displayed stock row.}
\label{tab:eval-2d}
\begin{tabular}{llrrrrrrr}
\toprule
suite & config & acc & $\Delta$acc & wall (s) & speedup & energy (Wh) & saved & $\Delta$power \\
\midrule
LongBench-v2 & stock & $50.50$ & --- & $2{,}970$ & --- & $3{,}326$ & --- & --- \\
 & $h{=}0,qh{=}1$ & $50.30$ & $-0.20$ & $3{,}076$ & $0.966\times$ & $3{,}354$ & $-0.8\%$ & $-2.6\%$ \\
 & $h{=}0,qh{=}2$ & $50.70$ & $+0.20$ & $3{,}017$ & $0.984\times$ & $3{,}339$ & $-0.4\%$ & $-1.2\%$ \\
 & $h{=}0,qh{=}3$ & $48.91$ & $-1.59$ & $3{,}031$ & $0.980\times$ & $3{,}341$ & $-0.4\%$ & $-1.6\%$ \\
 & $h{=}1,qh{=}1$ & $49.50$ & $-1.00$ & $2{,}832$ & $1.049\times$ & $3{,}039$ & $8.6\%$ & $-4.2\%$ \\
 & $h{=}1,qh{=}2$ & $49.50$ & $-1.00$ & $2{,}848$ & $1.043\times$ & $3{,}045$ & $8.4\%$ & $-4.5\%$ \\
 & $h{=}1,qh{=}3$ & $50.70$ & $+0.20$ & $2{,}838$ & $1.047\times$ & $3{,}041$ & $8.6\%$ & $-4.3\%$ \\
 & $h{=}2,qh{=}1$ & $48.91$ & $-1.59$ & $2{,}556$ & $1.162\times$ & $2{,}819$ & $15.3\%$ & $-1.5\%$ \\
 & $h{=}2,qh{=}2$ & $49.90$ & $-0.60$ & $2{,}558$ & $1.161\times$ & $2{,}815$ & $15.4\%$ & $-1.7\%$ \\
 & $h{=}2,qh{=}3$ & $49.70$ & $-0.80$ & $2{,}552$ & $1.164\times$ & $2{,}813$ & $15.4\%$ & $-1.6\%$ \\
 & $h{=}3,qh{=}1$ & $49.30$ & $-1.20$ & $2{,}484$ & $1.196\times$ & $2{,}759$ & $17.0\%$ & $-0.8\%$ \\
 & $h{=}3,qh{=}2$ & $48.71$ & $-1.79$ & $2{,}516$ & $1.181\times$ & $2{,}768$ & $16.8\%$ & $-1.8\%$ \\
 & $h{=}3,qh{=}3$ & $49.50$ & $-1.00$ & $2{,}523$ & $1.177\times$ & $2{,}770$ & $16.7\%$ & $-2.0\%$ \\
\bottomrule
\end{tabular}
\end{table*}

\begin{table*}[t]
\centering\footnotesize\setlength{\tabcolsep}{4pt}
\caption{Persistent decode reuse on RULER-1M ($qh\geq1$). The $h=0$ rows use query reuse alone; $qh=0$ results appear in Table~\ref{tab:eval-1d}. Ratios use the displayed stock row.}
\label{tab:eval-2d-ruler}
\begin{tabular}{llrrrrrrr}
\toprule
suite & config & acc & $\Delta$acc & wall (s) & speedup & energy (Wh) & saved & $\Delta$power \\
\midrule
RULER-1M & stock & $80.41$ & --- & $5{,}982$ & --- & $6{,}624$ & --- & --- \\
 & $h{=}0,qh{=}1$ & $78.41$ & $-2.00$ & $6{,}081$ & $0.984\times$ & $6{,}653$ & $-0.4\%$ & $-1.2\%$ \\
 & $h{=}0,qh{=}2$ & $82.51$ & $+2.10$ & $6{,}027$ & $0.992\times$ & $6{,}638$ & $-0.2\%$ & $-0.6\%$ \\
 & $h{=}0,qh{=}3$ & $79.85$ & $-0.56$ & $6{,}032$ & $0.992\times$ & $6{,}638$ & $-0.2\%$ & $-0.6\%$ \\
 & $h{=}1,qh{=}1$ & $80.46$ & $+0.05$ & $5{,}119$ & $1.168\times$ & $5{,}519$ & $16.7\%$ & $-2.7\%$ \\
 & $h{=}1,qh{=}2$ & $78.36$ & $-2.05$ & $5{,}119$ & $1.168\times$ & $5{,}514$ & $16.8\%$ & $-2.8\%$ \\
 & $h{=}1,qh{=}3$ & $75.69$ & $-4.72$ & $5{,}119$ & $1.168\times$ & $5{,}516$ & $16.7\%$ & $-2.7\%$ \\
 & $h{=}2,qh{=}1$ & $78.99$ & $-1.42$ & $4{,}493$ & $1.331\times$ & $4{,}964$ & $25.1\%$ & $-0.3\%$ \\
 & $h{=}2,qh{=}2$ & $74.71$ & $-5.70$ & $4{,}476$ & $1.336\times$ & $4{,}961$ & $25.1\%$ & $+0.1\%$ \\
 & $h{=}2,qh{=}3$ & $76.59$ & $-3.82$ & $4{,}479$ & $1.335\times$ & $4{,}963$ & $25.1\%$ & $+0.1\%$ \\
 & $h{=}3,qh{=}1$ & $74.97$ & $-5.44$ & $4{,}307$ & $1.389\times$ & $4{,}819$ & $27.2\%$ & $+1.0\%$ \\
 & $h{=}3,qh{=}2$ & $70.73$ & $-9.68$ & $4{,}316$ & $1.386\times$ & $4{,}820$ & $27.2\%$ & $+0.8\%$ \\
 & $h{=}3,qh{=}3$ & $74.23$ & $-6.18$ & $4{,}337$ & $1.379\times$ & $4{,}826$ & $27.1\%$ & $+0.5\%$ \\
\bottomrule
\end{tabular}
\end{table*}

\begin{table*}[t]
\centering\footnotesize\setlength{\tabcolsep}{4pt}
\caption{Persistent decode reuse on MRCR-v2 ($qh\geq1$). The $h=0$ rows use query reuse alone; $qh=0$ results appear in Table~\ref{tab:eval-1d}. Ratios use the displayed stock row.}
\label{tab:eval-2d-mrcr}
\begin{tabular}{llrrrrrrr}
\toprule
suite & config & acc & $\Delta$acc & wall (s) & speedup & energy (Wh) & saved & $\Delta$power \\
\midrule
MRCR-v2 & stock & $75.02$ & --- & $4{,}911$ & --- & $5{,}093$ & --- & --- \\
 & $h{=}0,qh{=}1$ & $73.70$ & $-1.32$ & $5{,}097$ & $0.964\times$ & $5{,}148$ & $-1.1\%$ & $-2.6\%$ \\
 & $h{=}0,qh{=}2$ & $70.48$ & $-4.54$ & $5{,}114$ & $0.960\times$ & $5{,}159$ & $-1.3\%$ & $-2.7\%$ \\
 & $h{=}0,qh{=}3$ & $72.24$ & $-2.78$ & $5{,}083$ & $0.966\times$ & $5{,}154$ & $-1.2\%$ & $-2.2\%$ \\
 & $h{=}1,qh{=}1$ & $73.88$ & $-1.14$ & $4{,}839$ & $1.015\times$ & $4{,}747$ & $6.8\%$ & $-5.4\%$ \\
 & $h{=}1,qh{=}2$ & $71.12$ & $-3.90$ & $4{,}872$ & $1.008\times$ & $4{,}759$ & $6.6\%$ & $-5.8\%$ \\
 & $h{=}1,qh{=}3$ & $68.29$ & $-6.73$ & $4{,}906$ & $1.001\times$ & $4{,}770$ & $6.3\%$ & $-6.2\%$ \\
 & $h{=}2,qh{=}1$ & $72.05$ & $-2.97$ & $4{,}453$ & $1.103\times$ & $4{,}435$ & $12.9\%$ & $-4.0\%$ \\
 & $h{=}2,qh{=}2$ & $74.06$ & $-0.96$ & $4{,}478$ & $1.097\times$ & $4{,}439$ & $12.8\%$ & $-4.4\%$ \\
 & $h{=}2,qh{=}3$ & $71.81$ & $-3.21$ & $4{,}445$ & $1.105\times$ & $4{,}429$ & $13.0\%$ & $-3.9\%$ \\
 & $h{=}3,qh{=}1$ & $71.73$ & $-3.29$ & $4{,}347$ & $1.130\times$ & $4{,}349$ & $14.6\%$ & $-3.5\%$ \\
 & $h{=}3,qh{=}2$ & $69.73$ & $-5.29$ & $4{,}421$ & $1.111\times$ & $4{,}371$ & $14.2\%$ & $-4.7\%$ \\
 & $h{=}3,qh{=}3$ & $69.85$ & $-5.17$ & $4{,}382$ & $1.121\times$ & $4{,}362$ & $14.3\%$ & $-4.0\%$ \\
\bottomrule
\end{tabular}
\end{table*}

\begin{table*}[t]
\centering\footnotesize\setlength{\tabcolsep}{4pt}
\caption{Persistent decode reuse on GraphWalks ($qh\geq1$). The $h=0$ rows use query reuse alone; $qh=0$ results appear in Table~\ref{tab:eval-1d}. Ratios use the displayed stock row.}
\label{tab:eval-2d-graphwalks}
\begin{tabular}{llrrrrrrr}
\toprule
suite & config & acc & $\Delta$acc & wall (s) & speedup & energy (Wh) & saved & $\Delta$power \\
\midrule
GraphWalks & stock & $53.35$ & --- & $12{,}264$ & --- & $12{,}976$ & --- & --- \\
 & $h{=}0,qh{=}1$ & $53.02$ & $-0.33$ & $12{,}269$ & $1.000\times$ & $12{,}987$ & $-0.1\%$ & $+0.1\%$ \\
 & $h{=}0,qh{=}2$ & $53.17$ & $-0.18$ & $12{,}262$ & $1.000\times$ & $13{,}004$ & $-0.2\%$ & $+0.2\%$ \\
 & $h{=}0,qh{=}3$ & $51.83$ & $-1.52$ & $12{,}363$ & $0.992\times$ & $13{,}005$ & $-0.2\%$ & $-0.6\%$ \\
 & $h{=}1,qh{=}1$ & $51.98$ & $-1.37$ & $10{,}700$ & $1.146\times$ & $11{,}111$ & $14.4\%$ & $-1.8\%$ \\
 & $h{=}1,qh{=}2$ & $52.05$ & $-1.30$ & $10{,}868$ & $1.129\times$ & $11{,}149$ & $14.1\%$ & $-3.0\%$ \\
 & $h{=}1,qh{=}3$ & $52.62$ & $-0.73$ & $10{,}783$ & $1.137\times$ & $11{,}115$ & $14.3\%$ & $-2.6\%$ \\
 & $h{=}2,qh{=}1$ & $51.50$ & $-1.85$ & $9{,}799$ & $1.252\times$ & $10{,}192$ & $21.5\%$ & $-1.7\%$ \\
 & $h{=}2,qh{=}2$ & $51.45$ & $-1.90$ & $9{,}894$ & $1.240\times$ & $10{,}195$ & $21.4\%$ & $-2.6\%$ \\
 & $h{=}2,qh{=}3$ & $51.98$ & $-1.37$ & $9{,}726$ & $1.261\times$ & $10{,}165$ & $21.7\%$ & $-1.2\%$ \\
 & $h{=}3,qh{=}1$ & $51.49$ & $-1.86$ & $9{,}396$ & $1.305\times$ & $9{,}914$ & $23.6\%$ & $-0.3\%$ \\
 & $h{=}3,qh{=}2$ & $50.34$ & $-3.01$ & $9{,}328$ & $1.315\times$ & $9{,}895$ & $23.7\%$ & $+0.3\%$ \\
 & $h{=}3,qh{=}3$ & $49.23$ & $-4.12$ & $9{,}273$ & $1.323\times$ & $9{,}867$ & $24.0\%$ & $+0.6\%$ \\
\bottomrule
\end{tabular}
\end{table*}

\paragraph{Coarse producer and group size.}
Table~\ref{tab:eval-cp} reports the mean-representative producer at
$P=2$, $b=8{,}192$, and $M=16{,}384$, with and without $qh=1$.
At $h=1,qh=0$, task-score changes are $-0.60$, $+0.28$, $-0.75$,
and $-0.01$ points across the four suites; adding $qh=1$ changes them
to $+0.79$, $+1.28$, $-1.13$, and $-0.13$.
At $h=2$, RULER-1M loses $5.71$ points without query reuse and $5.58$
with it. At $h=3$, MRCR-v2 loses $3.30$ and $3.36$ points.
The coarse producer is faster than its own stock row on most settings,
but the $h=1$ MRCR-v2 configuration is slower with either query setting.

Table~\ref{tab:eval-psweep} varies $P=2,4,8,16$ at $h=1,qh=0$ while
holding $bP=M$ fixed. Energy savings rise on all suites. RULER-1M score
changes are $+0.28$, $-4.08$, $-9.01$, and $-14.30$ points, while its
speedup rises from $1.229$ to $1.352\times$. MRCR-v2 loses $3.47$
points at $P=16$; GraphWalks loses $1.36$ at $P=8$ and $2.52$ at
$P=16$. GraphWalks BFS falls from $36.1$ for this campaign's stock row
to $31.9$ at $P=16$. LongBench-v2 shows no monotone score trend.
At $P\geq4$, mean GPU power falls by approximately $9$--$16\%$,
compared with approximately $5$--$10\%$ at $P=2$.

The reported RULER-1M single-needle tasks score $100$ throughout this
sweep. \texttt{niah\_multikey\_3} scores $50$, $10$, and $10$ at
$P=4,8,16$, while the multi-query and multi-value scores at $P=16$
are $78$ and $75$. This pattern motivates testing whether coarse
candidate construction loses important groups. It does not, by itself,
separate candidate loss from the earlier harness confound or explain
each downstream error.

\begin{table*}[t]
\centering\footnotesize\setlength{\tabcolsep}{4pt}
\caption{Coarse producer with mean representatives, $P=2$ and $b=8{,}192$. Ratios and score changes use this campaign's stock rows. Every configuration is one serving run.}
\label{tab:eval-cp}
\begin{tabular}{llrrrrrrr}
\toprule
suite & config & acc & $\Delta$acc & wall (s) & speedup & energy (Wh) & saved & $\Delta$power \\
\midrule
LongBench-v2 & stock & $50.70$ & --- & $2{,}990$ & --- & $3{,}339$ & --- & --- \\
 & $h{=}1$ & $50.10$ & $-0.60$ & $2{,}910$ & $1.027\times$ & $2{,}998$ & $10.2\%$ & $-7.8\%$ \\
 & $h{=}2$ & $50.70$ & $+0.00$ & $2{,}602$ & $1.149\times$ & $2{,}793$ & $16.3\%$ & $-3.9\%$ \\
 & $h{=}3$ & $49.90$ & $-0.80$ & $2{,}590$ & $1.154\times$ & $2{,}736$ & $18.1\%$ & $-5.4\%$ \\
 & $h{=}1,qh{=}1$ & $51.49$ & $+0.79$ & $2{,}956$ & $1.012\times$ & $2{,}988$ & $10.5\%$ & $-9.5\%$ \\
 & $h{=}2,qh{=}1$ & $51.49$ & $+0.79$ & $2{,}646$ & $1.130\times$ & $2{,}780$ & $16.7\%$ & $-5.9\%$ \\
 & $h{=}3,qh{=}1$ & $48.71$ & $-1.99$ & $2{,}572$ & $1.163\times$ & $2{,}730$ & $18.3\%$ & $-5.0\%$ \\
\midrule
RULER-1M & stock & $79.13$ & --- & $5{,}990$ & --- & $6{,}644$ & --- & --- \\
 & $h{=}1$ & $79.41$ & $+0.28$ & $4{,}876$ & $1.229\times$ & $5{,}147$ & $22.5\%$ & $-4.8\%$ \\
 & $h{=}2$ & $73.42$ & $-5.71$ & $4{,}283$ & $1.399\times$ & $4{,}723$ & $28.9\%$ & $-0.6\%$ \\
 & $h{=}3$ & $78.56$ & $-0.57$ & $4{,}247$ & $1.411\times$ & $4{,}609$ & $30.6\%$ & $-2.2\%$ \\
 & $h{=}1,qh{=}1$ & $80.41$ & $+1.28$ & $4{,}935$ & $1.214\times$ & $5{,}116$ & $23.0\%$ & $-6.5\%$ \\
 & $h{=}2,qh{=}1$ & $73.55$ & $-5.58$ & $4{,}320$ & $1.387\times$ & $4{,}692$ & $29.4\%$ & $-2.1\%$ \\
 & $h{=}3,qh{=}1$ & $78.46$ & $-0.67$ & $4{,}212$ & $1.422\times$ & $4{,}587$ & $31.0\%$ & $-1.8\%$ \\
\midrule
MRCR-v2 & stock & $74.58$ & --- & $4{,}980$ & --- & $5{,}125$ & --- & --- \\
 & $h{=}1$ & $73.83$ & $-0.75$ & $5{,}089$ & $0.979\times$ & $4{,}729$ & $7.7\%$ & $-9.7\%$ \\
 & $h{=}2$ & $73.45$ & $-1.13$ & $4{,}642$ & $1.073\times$ & $4{,}438$ & $13.4\%$ & $-7.1\%$ \\
 & $h{=}3$ & $71.28$ & $-3.30$ & $4{,}535$ & $1.098\times$ & $4{,}339$ & $15.3\%$ & $-7.0\%$ \\
 & $h{=}1,qh{=}1$ & $73.45$ & $-1.13$ & $5{,}136$ & $0.970\times$ & $4{,}694$ & $8.4\%$ & $-11.2\%$ \\
 & $h{=}2,qh{=}1$ & $73.65$ & $-0.93$ & $4{,}663$ & $1.068\times$ & $4{,}393$ & $14.3\%$ & $-8.4\%$ \\
 & $h{=}3,qh{=}1$ & $71.22$ & $-3.36$ & $4{,}550$ & $1.095\times$ & $4{,}324$ & $15.6\%$ & $-7.7\%$ \\
\midrule
GraphWalks & stock & $53.82$ & --- & $12{,}094$ & --- & $12{,}976$ & --- & --- \\
 & $h{=}1$ & $53.81$ & $-0.01$ & $11{,}125$ & $1.087\times$ & $10{,}834$ & $16.5\%$ & $-9.2\%$ \\
 & $h{=}2$ & $51.10$ & $-2.72$ & $9{,}787$ & $1.236\times$ & $9{,}843$ & $24.1\%$ & $-6.3\%$ \\
 & $h{=}3$ & $53.83$ & $+0.01$ & $9{,}670$ & $1.251\times$ & $9{,}688$ & $25.3\%$ & $-6.6\%$ \\
 & $h{=}1,qh{=}1$ & $53.69$ & $-0.13$ & $11{,}120$ & $1.088\times$ & $10{,}702$ & $17.5\%$ & $-10.3\%$ \\
 & $h{=}2,qh{=}1$ & $52.54$ & $-1.28$ & $9{,}649$ & $1.253\times$ & $9{,}799$ & $24.5\%$ & $-5.3\%$ \\
 & $h{=}3,qh{=}1$ & $54.08$ & $+0.26$ & $9{,}500$ & $1.273\times$ & $9{,}649$ & $25.6\%$ & $-5.3\%$ \\
\bottomrule
\end{tabular}
\end{table*}

\begin{table*}[t]
\centering\footnotesize\setlength{\tabcolsep}{4pt}
\caption{Coarse-producer group-size sweep at $h=1$, $qh=0$, and fixed $bP=M=16{,}384$. The $P=2$ rows repeat Table~\ref{tab:eval-cp}; all ratios use the displayed stock rows.}
\label{tab:eval-psweep}
\begin{tabular}{llrrrrrrr}
\toprule
suite & config & acc & $\Delta$acc & wall (s) & speedup & energy (Wh) & saved & $\Delta$power \\
\midrule
LongBench-v2 & stock & $50.70$ & --- & $2{,}990$ & --- & $3{,}339$ & --- & --- \\
 & $P{=}2$ & $50.10$ & $-0.60$ & $2{,}910$ & $1.027\times$ & $2{,}998$ & $10.2\%$ & $-7.8\%$ \\
 & $P{=}4$ & $52.49$ & $+1.79$ & $2{,}878$ & $1.039\times$ & $2{,}828$ & $15.3\%$ & $-12.0\%$ \\
 & $P{=}8$ & $49.90$ & $-0.80$ & $2{,}786$ & $1.073\times$ & $2{,}722$ & $18.5\%$ & $-12.5\%$ \\
 & $P{=}16$ & $51.09$ & $+0.39$ & $2{,}768$ & $1.080\times$ & $2{,}680$ & $19.7\%$ & $-13.3\%$ \\
\midrule
RULER-1M & stock & $79.13$ & --- & $5{,}990$ & --- & $6{,}644$ & --- & --- \\
 & $P{=}2$ & $79.41$ & $+0.28$ & $4{,}876$ & $1.229\times$ & $5{,}147$ & $22.5\%$ & $-4.8\%$ \\
 & $P{=}4$ & $75.05$ & $-4.08$ & $4{,}663$ & $1.285\times$ & $4{,}695$ & $29.3\%$ & $-9.2\%$ \\
 & $P{=}8$ & $70.12$ & $-9.01$ & $4{,}458$ & $1.344\times$ & $4{,}449$ & $33.0\%$ & $-10.0\%$ \\
 & $P{=}16$ & $64.83$ & $-14.30$ & $4{,}432$ & $1.352\times$ & $4{,}347$ & $34.6\%$ & $-11.6\%$ \\
\midrule
MRCR-v2 & stock & $74.58$ & --- & $4{,}980$ & --- & $5{,}125$ & --- & --- \\
 & $P{=}2$ & $73.83$ & $-0.75$ & $5{,}089$ & $0.979\times$ & $4{,}729$ & $7.7\%$ & $-9.7\%$ \\
 & $P{=}4$ & $74.54$ & $-0.04$ & $4{,}984$ & $0.999\times$ & $4{,}455$ & $13.1\%$ & $-13.1\%$ \\
 & $P{=}8$ & $74.75$ & $+0.17$ & $4{,}845$ & $1.028\times$ & $4{,}301$ & $16.1\%$ & $-13.7\%$ \\
 & $P{=}16$ & $71.11$ & $-3.47$ & $4{,}910$ & $1.014\times$ & $4{,}268$ & $16.7\%$ & $-15.5\%$ \\
\midrule
GraphWalks & stock & $53.82$ & --- & $12{,}094$ & --- & $12{,}976$ & --- & --- \\
 & $P{=}2$ & $53.81$ & $-0.01$ & $11{,}125$ & $1.087\times$ & $10{,}834$ & $16.5\%$ & $-9.2\%$ \\
 & $P{=}4$ & $55.07$ & $+1.25$ & $10{,}398$ & $1.163\times$ & $9{,}870$ & $23.9\%$ & $-11.5\%$ \\
 & $P{=}8$ & $52.46$ & $-1.36$ & $10{,}148$ & $1.192\times$ & $9{,}431$ & $27.3\%$ & $-13.4\%$ \\
 & $P{=}16$ & $51.30$ & $-2.52$ & $9{,}825$ & $1.231\times$ & $9{,}185$ & $29.2\%$ & $-12.9\%$ \\
\bottomrule
\end{tabular}
\end{table*}

\paragraph{Representative operator.}
Table~\ref{tab:eval-rep} replaces a group's mean by its elementwise maximum
at $P=2$. At $h=1$, max-representative speedups against the displayed stock
rows are $1.005$, $1.203$, $0.954$, and $1.072\times$, compared with
$1.027$, $1.229$, $0.979$, and $1.087\times$ for the separately measured
mean campaign. The GPU energy savings are $10.4$, $22.7$, $7.7$, and
$17.1\%$, compared with $10.2$, $22.5$, $7.7$, and $16.5\%$.
These are campaign-specific ratios rather than a repeated paired
comparison of the two operators.

The max operator has RULER-1M score changes of $-1.59$, $-9.23$, and
$-2.14$ points at $h=1,2,3$, versus $+0.28$, $-5.71$, and $-0.57$
for the mean. At $h=2$, \texttt{niah\_multikey\_2} and
\texttt{niah\_multikey\_3} score $70$ and $50$ with the mean, compared
with $70$ and $10$ with the max. The source reports every other task
within one prompt of stock under both operators. Neither operator
eliminates the difficult RULER-1M configuration. Repeated comparisons
with verified matched controls are needed to rank them confidently.

\begin{table*}[t]
\centering\footnotesize\setlength{\tabcolsep}{4pt}
\caption{Coarse producer with elementwise-max representatives, $P=2$ and $b=8{,}192$. Ratios and score changes use the displayed stock rows, also used in Table~\ref{tab:eval-1d}.}
\label{tab:eval-rep}
\begin{tabular}{llrrrrrrr}
\toprule
suite & config & acc & $\Delta$acc & wall (s) & speedup & energy (Wh) & saved & $\Delta$power \\
\midrule
LongBench-v2 & stock & $50.50$ & --- & $2{,}970$ & --- & $3{,}326$ & --- & --- \\
 & max, $h{=}1$ & $50.30$ & $-0.20$ & $2{,}956$ & $1.005\times$ & $2{,}979$ & $10.4\%$ & $-10.0\%$ \\
 & max, $h{=}2$ & $51.29$ & $+0.79$ & $2{,}657$ & $1.118\times$ & $2{,}782$ & $16.3\%$ & $-6.5\%$ \\
 & max, $h{=}3$ & $49.70$ & $-0.80$ & $2{,}616$ & $1.136\times$ & $2{,}738$ & $17.7\%$ & $-6.5\%$ \\
\midrule
RULER-1M & stock & $80.41$ & --- & $5{,}982$ & --- & $6{,}624$ & --- & --- \\
 & max, $h{=}1$ & $78.82$ & $-1.59$ & $4{,}974$ & $1.203\times$ & $5{,}118$ & $22.7\%$ & $-7.1\%$ \\
 & max, $h{=}2$ & $71.18$ & $-9.23$ & $4{,}378$ & $1.366\times$ & $4{,}699$ & $29.1\%$ & $-3.1\%$ \\
 & max, $h{=}3$ & $78.27$ & $-2.14$ & $4{,}316$ & $1.386\times$ & $4{,}610$ & $30.4\%$ & $-3.5\%$ \\
\midrule
MRCR-v2 & stock & $75.02$ & --- & $4{,}911$ & --- & $5{,}093$ & --- & --- \\
 & max, $h{=}1$ & $75.59$ & $+0.57$ & $5{,}150$ & $0.954\times$ & $4{,}702$ & $7.7\%$ & $-12.0\%$ \\
 & max, $h{=}2$ & $72.36$ & $-2.66$ & $4{,}638$ & $1.059\times$ & $4{,}386$ & $13.9\%$ & $-8.8\%$ \\
 & max, $h{=}3$ & $73.10$ & $-1.92$ & $4{,}541$ & $1.081\times$ & $4{,}311$ & $15.3\%$ & $-8.4\%$ \\
\midrule
GraphWalks & stock & $53.35$ & --- & $12{,}264$ & --- & $12{,}976$ & --- & --- \\
 & max, $h{=}1$ & $53.65$ & $+0.30$ & $11{,}445$ & $1.072\times$ & $10{,}761$ & $17.1\%$ & $-11.1\%$ \\
 & max, $h{=}2$ & $50.92$ & $-2.43$ & $9{,}765$ & $1.256\times$ & $9{,}799$ & $24.5\%$ & $-5.1\%$ \\
 & max, $h{=}3$ & $53.71$ & $+0.36$ & $9{,}509$ & $1.290\times$ & $9{,}631$ & $25.8\%$ & $-4.3\%$ \\
\bottomrule
\end{tabular}
\end{table*}

\clearpage
% Source fragment: sections/appendix_topg.tex
\section{Exploratory Top-\texorpdfstring{$G$}{G} Construction and Trace Study}
\label{app:topg}

\subsection{Capture and measurement protocol}
\label{sec:topg-method}

The Top-$G$ study uses LB2 traces captured while running MLIPS. The capture
contains eight sequences and 21 compressed sparse-attention (CSA) layers
at physical layers $2,4,\ldots,42$. One CSA-layer step therefore spans two
physical model layers. Each recorded query supplies unordered producer
Top-$K$ and Top-$M$ index sets, with $K=512$ and nominal
$M=16\mathrm{K}$. The dump lacks aligned producer-layer scores, so the
constructor below uses set membership. Array position within a captured
Top-$M$ set is not treated as an importance ranking. Pass identifiers are
trace provenance and do not affect construction.

Four consecutive queries are recorded every 256 original query positions.
Groups of 64, 256, and 1,024 recorded rows consequently represent nominal
windows of 4,096, 16,384, and 65,536 original positions. Sequence-wide
sharing uses all eligible recorded rows in a sequence. These labels
describe sampling spans, not fully observed query groups. We cannot infer
how recall changes when all intervening queries participate: additional
queries may provide useful evidence but also compete for the same budget.

All reported pools have budget $G=32{,}000$. Measurements require context
length at least $G$ and average recall on capacity-eligible recorded rows.
The primary aggregate in Figure~\ref{fig:topg-recall} covers sequences
0--6 of the eight captured sequences and subsequent CSA-layer distances
1--6. Producer-layer fit, at distance zero, is excluded from reuse results.
The headline results are restricted to these seven sequences.
Per-sequence results below use the 11 producer layers with a complete
ten-step consumer horizon. Each sequence--window--distance cell is a
query-cell-weighted mean over eligible recorded rows and those producers.

For a recorded query $q$, write $K_q^\ell$ and $M_q^\ell$ for producer
memberships. The pool satisfies Equation~\ref{eq:topg-pool}, and evaluation
uses Equation~\ref{eq:topg-recall}. Recall is defined only for nonempty
consumer selections. The pool is frozen before examining any consumer
layer; consumer Top-$K$ labels are used only for evaluation.

\subsection{Coverage-debt constructor}

The constructor assigns more weight to queries whose producer selections
remain uncovered. Define the residual producer membership
$A_q=M_q^\ell\setminus K_q^\ell$ and, for a partially filled pool
$\mathcal{G}$, define
\begin{equation}
 D(S;\mathcal{G})=
 \begin{cases}
  1-|S\cap\mathcal{G}|/|S|,& |S|>0,\\
  0,& |S|=0.
 \end{cases}
 \label{eq:coverage-debt}
\end{equation}
Let $d_q^K=D(K_q^\ell;\mathcal{G})$ and
$d_q^A=D(A_q;\mathcal{G})$. An unselected candidate $x$ receives primary
and secondary scores
\begin{align}
 P(x)&=\sum_{q\in C:\,|K_q^\ell|>0}
  \frac{d_q^K}{|K_q^\ell|}\,\mathbf{1}[x\in K_q^\ell],
 \label{eq:primary-score}\\
 S(x)&=\sum_{q\in C:\,|A_q|>0}
  \frac{d_q^A}{|A_q|}\,\mathbf{1}[x\in A_q].
 \label{eq:secondary-score}
\end{align}
Zero-size memberships contribute no vote, avoiding undefined
normalization. Candidates are ranked lexicographically by descending
$(P(x),S(x),\tau(x))$, where $\tau$ is a fixed seeded key-ID hash used
for deterministic tie breaking; key ID resolves any remaining hash tie.
Thus residual Top-$M$ evidence never overrides a larger primary score.
After each batch of 1,024 selected keys, debts and scores are recomputed.
This is a batched coverage heuristic, not an optimization guarantee.

\begin{algorithm}[t]
 \caption{Coverage-debt pool construction}
 \label{alg:coverage-debt}
 \begin{algorithmic}[1]
  \REQUIRE Recorded sets $\{K_q^\ell,M_q^\ell:q\in C\}$,
    capacity $G$, batch $B=1{,}024$
  \STATE $\mathcal{G}\leftarrow\varnothing$
  \STATE $A_q\leftarrow M_q^\ell\setminus K_q^\ell$ for every $q\in C$
  \STATE $U\leftarrow\bigcup_{q\in C}(K_q^\ell\cup M_q^\ell)$
  \WHILE{$|\mathcal{G}|<\min(G,|U|)$}
   \STATE Compute debts using Equation~\ref{eq:coverage-debt}
   \STATE Compute $P(x),S(x)$ for $x\in U\setminus\mathcal{G}$
   \STATE $b\leftarrow\min(B,G-|\mathcal{G}|,|U\setminus\mathcal{G}|)$
   \STATE $V\leftarrow$ first $b$ candidates in descending $(P,S,\tau)$ order
   \STATE $\mathcal{G}\leftarrow\mathcal{G}\cup V$
  \ENDWHILE
  \STATE \textbf{return} $\mathcal{G}$
 \end{algorithmic}
\end{algorithm}

As a query becomes covered, its primary vote shrinks and the budget shifts
toward other queries. This construction supplies a concrete example of how
producer memberships can reveal a shared candidate working set. It uses no
consumer-layer labels and no trained predictor. That restriction concerns
the layer axis: constructing the pool from all queries in a group can still
make the candidate set for an earlier query depend on later queries.
Strict prefix-invariant inference therefore requires an explicit
construction schedule and causal masks, and may require a different
constructor. Masking future key positions alone does not remove dependence
of pool membership on future query evidence.

\subsection{Aggregate and per-sequence evidence}

For a fixed budget of 32,000 indices, the mean recall across distances
1--6 is $77.6\%$, $69.4\%$, $52.9\%$, and $33.6\%$ for nominal windows
of 4,096, 16,384, 65,536, and the full sequence, respectively
(Figure~\ref{fig:topg-recall}). The first-to-sixth-distance decrease is
$82.4\%$ to $73.5\%$ for the smallest window. At 65,536 positions,
the sampling span is sixteen times larger while mean recall remains
$52.9\%$. These results reveal reuse across queries and layers, together
with a clear capacity boundary as the query scope expands.

The ordering also holds separately for each primary sequence, 0--6:
the mean over ten consumer distances decreases as the nominal window
grows. For the 4,096-position window, per-sequence ten-distance mean
recall ranges from $65.3\%$ to $84.1\%$, and recall at distance ten
ranges from $59.9\%$ to $80.6\%$. The aggregate trend is consequently
not confined to one input, although the coverage level remains workload
dependent.

\begin{figure*}[t]
\centering
\includegraphics[width=\textwidth]{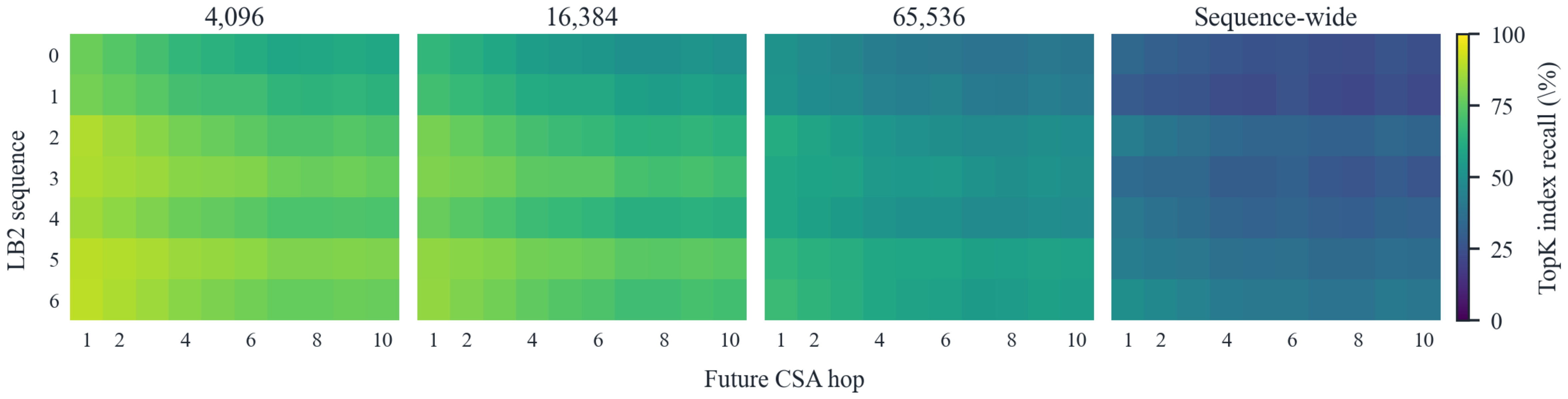}
\caption{Per-sequence Top-$G$ coverage at $G=32{,}000$. Rows are LB2 sequences
0--6 and columns are subsequent CSA-layer distances; panels vary the nominal
query window. Cells average eligible recorded queries over the 11 producers
with a complete ten-layer horizon. The original plot uses \emph{hop} for exact
layer distance $d$.}
\label{fig:topg-sequences}
\end{figure*}

\subsection{What remains to be established}

The trace study measures coverage of recorded selected indices. It does
not evaluate every query, end-task accuracy, construction overhead,
consumer selection cost, or serving performance. The constructor is an
exploratory probe rather than a validated runtime policy. Comparisons
against producer Top-$K$ frequency, fixed Top-$K$/Top-$M$ quotas, private
producer Top-$M$ pools, and a future-selection oracle would help isolate
which parts of the construction explain coverage; the present aggregate
alone does not establish their relative merit.

Shared candidate metadata and coordinated KV staging are potential
applications, but a pool of $G$ entries may exceed each query's original
Top-$M$ budget and increase consumer scoring. A shared index set also does
not imply that different layers share KV values. Realizing a residency
benefit requires layer-specific data movement and memory measurements.
Pool refresh, query-group boundaries, causal scheduling, and any recovery
for missed candidates remain deployment choices to evaluate.

\clearpage
% Source fragment: sections/appendix_offload.tex
\section{Candidate Working Sets and Offloading}
\label{app:offload}

Sharing candidate indices does not reduce the total KV state that an autoregressive
model must preserve. A future pool may contain entries absent from the current one.
Those entries must remain recoverable, for example in host memory, if they are evicted
from the GPU. This section analyzes a possible resident working set, not an implemented
or measured memory-saving mechanism.

\paragraph{A simplified capacity model.}
Suppose each producer keeps all $N$ compressed positions resident, and each consumer
keeps the indexer keys and attention KV for its current $M$ candidates resident.
Under that policy, the compressed-entry working set over $n$ layers is
$pN+cM$, giving the modeled ratio
\begin{equation}
R_{\mathrm{resident}}=\frac{nN}{pN+cM}.
\label{eq:offload-peak}
\end{equation}
This equals the layerwise dot-product reduction because both counts assign $N$
positions to a producer and $M$ to a consumer. The identity is a property of this
simplified policy, not an equivalence between compute savings and measured GPU memory.
Attention itself reads only the selected top-$K$; alternative staging policies could
use a different resident set.

If a producer also keeps only $M$ entries during its $qh$ reuse steps, the
cycle-average ratio becomes
\begin{equation}
R_{\mathrm{avg}}=
\frac{nN(qh+1)}{pN+[p\,qh+c(qh+1)]M}.
\label{eq:offload-avg}
\end{equation}
For $qh>0$, this is a time-average, not a reduction in the capacity needed for a
rebuild step. The values match Table~\ref{tab:ops-2d}, so we do not repeat that table.
Both models exclude weights, local windows, other attention layers, metadata,
transfer buffers, and runtime allocation overhead.

\paragraph{Transfers and deadlines.}
Consumer keys that enter a refreshed pool must be fetched before re-scoring;
the selected KV entries must arrive before attention. The pool becomes known
only when the producer publishes it, leaving a short window before the first
consumer. Persistent pools can lengthen this window on later reuse steps.
Achieving the producer's cycle-average working set additionally requires
reloading its full search domain before the next rebuild.

Pool turnover, transfer bandwidth, staging buffers, and scheduling determine whether
the policy is practical. The current top-$K$-in-pool recall measurements do not measure
pool-to-pool turnover, so they cannot quantify this transfer volume. An offloading
evaluation should report actual peak GPU residency, transferred bytes, latency,
GPU energy, and task quality together.

\clearpage
\end{document}